\documentclass[lettersize,journal]{IEEEtran}
\usepackage{amsmath,amsfonts}
\usepackage{algorithmic}
\usepackage[ruled]{algorithm2e}
\usepackage{array}
\usepackage[caption=false,font=normalsize,labelfont=sf,textfont=sf]{subfig}
\usepackage{textcomp}
\usepackage{stfloats}
\usepackage{url}
\usepackage{verbatim}
\usepackage{graphicx}
\usepackage{cite}
\usepackage{booktabs} 
\usepackage{cleveref}
\usepackage{multirow}
\usepackage{amssymb}
\usepackage{color}

\newcommand{\dcsgdp}{\texttt{DC-SGD-P}\xspace}
\newcommand{\dcsgde}{\texttt{DC-SGD-E}\xspace}
\newcommand{\dcsgd}{\texttt{DC-SGD}\xspace}

\newcommand{\dpsgd}{{DP-SGD}\xspace}

\renewcommand{\smallskip}{\vspace{2pt plus 0.5pt minus 0.5pt}}

\newtheorem{definition}{\bf Definition}
\newtheorem{theorem}{\bf Theorem}
\newtheorem{lemma}{\bf Lemma}

\usepackage{xparse}
\newcommand{\bnm}{\begin{newmath}}
\newcommand{\enm}{\end{newmath}}

\newcommand{\bea}{\begin{eqnarray*}}%
\newcommand{\eea}{\end{eqnarray*}}%

\newcommand{\bne}{\begin{newequation}}
\newcommand{\ene}{\end{newequation}}

\newcommand{\bal}{\begin{newalign}}
\newcommand{\eal}{\end{newalign}}

\newenvironment{newalign}{\begin{align}%
\setlength{\abovedisplayskip}{4pt}%
\setlength{\belowdisplayskip}{4pt}%
\setlength{\abovedisplayshortskip}{6pt}%
\setlength{\belowdisplayshortskip}{6pt} }{\end{align}}

\newenvironment{newmath}{\begin{displaymath}%
\setlength{\abovedisplayskip}{4pt}%
\setlength{\belowdisplayskip}{4pt}%
\setlength{\abovedisplayshortskip}{6pt}%
\setlength{\belowdisplayshortskip}{6pt} }{\end{displaymath}}

\newenvironment{newequation}{\begin{equation}%
\setlength{\abovedisplayskip}{4pt}%
\setlength{\belowdisplayskip}{4pt}%
\setlength{\abovedisplayshortskip}{6pt}%
\setlength{\belowdisplayshortskip}{6pt} }{\end{equation}}

\newcounter{ctr}

\newcounter{mytable}
\def\mytable{\begin{centering}\refstepcounter{mytable}}
\def\endmytable{\end{centering}}

\newcounter{myfig}
\def\myfig{\begin{centering}\refstepcounter{myfig}}
\def\endmyfig{\end{centering}}

\newlength{\saveparindent}
\newlength{\saveparskip}
\newcommand{\E}{{\rm I\kern-.3em E}}

\newcommand{\assref}[1]{\mbox{Assumption~\ref{#1}}}
\newcommand{\secref}[1]{\mbox{Section~\ref{#1}}}
\newcommand{\appref}[1]{\mbox{Appendix~\ref{#1}}}

\renewcommand{\eqref}[1]{\mbox{Equation~(\ref{#1})}}

\newcommand{\get}{{\:{\leftarrow}\:}}

\def \part {part}

\renewcommand{\paragraph}[1]{\vspace*{6pt}\noindent\textbf{#1}\;}

\newtheorem{assumption}{Assumption}

\def \blackslug{\hbox{\hskip 1pt \vrule width 4pt height 8pt
    depth 1.5pt \hskip 1pt}}
\def \qed{\quad\blackslug\lower 8.5pt\null\par}

\newcounter{mynote}[section]

\newcommand\ignore[1]{}

\newcounter{rcnote}[section]

\newcounter{mrnote}[section]

\newcounter{fknote}[section]

\newcounter{anote}[section]

\DeclareMathSymbol{\mlq}{\mathord}{operators}{``}
\DeclareMathSymbol{\mrq}{\mathord}{operators}{`'}

\newcommand{\rhf}[2]{R_{f, \gamma}}

\DeclareDocumentCommand{\edist}{o o}{
  \ensuremath{
    \IfNoValueTF{#1}{{d}}{{\sf d}(#1,#2)}
  }
}

\newcommand{\olrk}[1]{\ifx\nursymbol#1\else\!\!\mskip4.5mu plus 0.5mu\left(\mskip0.5mu plus0.5mu #1\mskip1.5mu plus0.5mu \right)\fi}

\def \q{q}

\NewDocumentCommand{\indseq}{ O{1} O{r} }{{#1}\ldots {#2}}

\begin{document}

\title{DC-SGD: Differentially Private SGD with Dynamic Clipping through Gradient Norm Distribution Estimation}

\author{{Chengkun Wei, Weixian Li, Chen Gong, and Wenzhi Chen}
\thanks{Chengkun Wei and Wenzhi Chen are with the College of Computer Science and Technology, Zhejiang University, Hangzhou 310027, China (e-mail: \{weichengkun, chenwz\}@zju.edu.cn).}
\thanks{Weixian Li is with the Ant Group (e-mail: liweixian.lwx@antgroup.com)}
\thanks{Chen Gong is with the School of Engineering, University of Virginia (e-mail:fzv6en@virginia.edu).}
}

\maketitle

\begin{abstract}
Differentially Private Stochastic Gradient Descent (DP-SGD) is a widely adopted technique for privacy-preserving deep learning. A critical challenge in DP-SGD is selecting the optimal clipping threshold $C$, which involves balancing the trade-off between clipping bias and noise magnitude, incurring substantial privacy and computing overhead during hyperparameter tuning.

In this paper, we propose Dynamic Clipping DP-SGD (DC-SGD), a framework that leverages differentially private histograms to estimate gradient norm distributions and dynamically adjust the clipping threshold $C$.
Our framework includes two novel mechanisms: DC-SGD-P and DC-SGD-E. DC-SGD-P adjusts the clipping threshold based on a percentile of gradient norms, while DC-SGD-E minimizes the expected squared error of gradients to optimize 
$C$. These dynamic adjustments significantly reduce the burden of hyperparameter tuning $C$.
The extensive experiments on various deep learning tasks, including image classification and natural language processing, show that our proposed dynamic algorithms achieve up to 9 times acceleration on hyperparameter tuning than \dpsgd.
And DC-SGD-E can achieve an accuracy improvement of $10.62\%$ on CIFAR10 than \dpsgd under the same privacy budget of hyperparameter tuning.
We conduct rigorous theoretical privacy and convergence analyses, showing that our methods seamlessly integrate with the Adam optimizer.
Our results highlight the robust performance and efficiency of DC-SGD, offering a practical solution for differentially private deep learning with reduced computational overhead and enhanced privacy guarantees.

\end{abstract}

\maketitle

\section{Introduction}
\label{sec:intro}

\IEEEPARstart{D}{eep} learning (DL) has made significant strides in various fields, such as computer vision~\cite{Zhao_2023_ICCV,cv1,cv2,cv3,cv4} and natural language processing~\cite{otter2020survey,nlp1,nlp2,nlp3}.
However, training deep learning models using sensitive data, like images or DNA sequences, can pose serious privacy challenges, potentially exposing sensitive details~\cite{shokri2017membership,choquette2021label,hayes2017logan,fredrikson2015model,gong2022curiosity,li2023meticulously,gong2020identification,pang2023white}. Differentially Private Stochastic Gradient Descent (\dpsgd)~\cite{abadi2016deep}, which clips gradients and add noise to the clipped gradients to satisfy Differentially Privacy (DP), has become one of the most promising solutions for deep learning privacy.

The core principle of DP-SGD is to clip gradients and add noise to ensure DP. 
However, selecting an appropriate hyper-parameters clipping threshold $C$ remains a knotty issue \cite{andrew2021differentially,de2022unlocking,bu2022automatic,golatkar2022mixed}. 
This difficulty arises from the fact that the optimal threshold varies across different scenarios ~\cite{du2021dynamic,chen2020understanding}.
Specifically, the optimal value for $C$ can change throughout the training process. Initially, a higher $C$ might be beneficial to limit noise, but as training progresses and gradients typically become smaller, the same $C$ could result in minimal clipping and an overestimation of noise impact. 
This trade-off between clipping and noise is critical: a higher $C$ reduces clipping but increases noise, potentially degrading model accuracy, while a lower $C$ reduces noise but increases bias, leading to information loss and reduced performance.

Moreover, \textbf{tuning the clipping threshold involves substantial computational overhead and can result in additional privacy leakage~\cite{liu2019private,papernot2021hyperparameter}.} The hyperparameter tuning costs are significant. This process often requires multiple training runs with different $C$ values to identify the optimal setting, each run being time-consuming and costly in terms of privacy. This highlights the importance of developing methods that can dynamically adjust $C$ without extensive manual tuning.

\noindent \textbf{Existing Solutions.} 
Prior  research~\cite{abadi2016deep, pichapati2019adaclip,du2021dynamic, dong2019gaussian} has investigated the use of dynamic clipping thresholds to mitigate noise addition.
For example, Abadi et al.~\cite{abadi2016deep} recommend setting the clipping threshold, $C$, as the median of gradient norms, and Du et al.~\cite{du2021dynamic} advocate for adjusting the threshold according to the training iteration. However, these methods either lack definitive privacy guarantees or require careful initial settings for the clipping threshold and its update rate~\cite{abadi2016deep,du2021dynamic}.
Another promising research direction involves dynamically setting the clipping threshold $C$ based on the distribution of gradient norms. Andrew et al.~\cite{andrew2021differentially} propose determining $C$ using a hyperparameter $p$, representing a percentile of gradient norm distribution in federated learning with DP-FedAvg. While this approach eliminates the need for manual tuning of $C$, it introduces other hyperparameters, $p$ and $\eta_C$, thus still incurring more hyperparameter tuning costs.

In this paper, we aim to solve the problem:
\textit{
How can we update the clipping threshold through private gradient norm distribution estimation, eliminating the need for hyperparameter tuning of $C$?}

\noindent \textbf{Our method.} We present \textbf{D}ynamic \textbf{C}lipping DP-\textbf{SGD} (\dcsgd), a framework that dynamically updates the clipping threshold in differentially private SGD. Distinguishing from previous research, \dcsgd investigates the intricate relationship between the distribution of gradient norms and the clipping threshold, initially accessing the gradient norm distribution and then adaptively setting the clipping threshold. To ensure privacy, we use a histogram to estimate the distribution of gradient norms in a differentially private manner. Within this framework, \dcsgd is implemented through two mechanisms, \dcsgdp and \dcsgde, which are elaborated as follows:

We first propose the mechanism \dcsgd with Percentile (\dcsgdp), which is similar to Andrew et al.~\cite{andrew2021differentially}. However, it has fewer parameters that need to be tuned.
We adaptively set the clipping threshold during training to ensure that $p\%$ of gradients remain unclipped. In this way, we can enhance the stability of clipping by maintaining a consistent probability that each gradient is clipped throughout the training process.
We shift the tuning from $C$ to $p$, as our experimental findings suggest that the selection of P adheres to discernible patterns. What's more, we eliminate the cost of tuning $\eta_C$ by adaptively setting $C$ with the histogram.

The second mechanism is \dcsgd with Expected Squared Error (\dcsgde). We explore a noisy histogram to estimate this error under differential privacy, and dynamically construct a candidate set for the optimal clipping threshold. We dynamically set the clipping threshold by optimizing the expected squared error between the noisy and original gradients in SGD. By minimizing the expected squared error, we can reduce the impact of the privacy operation on individual gradients. Notably, \textit{\dcsgde requires no additional hyperparameters to be tuned, thus eliminating the need for tuning $C$}.

We highlight the robust transferability of \dcsgd, as our methods are compatible with the Adam optimizer~\cite{kingma2014adam}. The Adam optimizer is particularly advantageous because it minimizes the need for tuning the learning rate ($\eta$). By integrating our methods with the Adam optimizer, we significantly reduce the complexity of simultaneously adjusting both the clipping threshold and the learning rate $\eta$. This integration enhances the overall efficiency and usability of our approach.

\smallskip
\noindent \textbf{Evaluations.} We perform a rigorous privacy analysis of \dcsgdp and \dcsgde by converting it into equivalent \dpsgd, and provide rigorous convergence analysis.
We also conduct comprehensive experiments to evaluate the effectiveness of \dcsgd across four datasets: MNIST~\cite{lecun-mnisthandwrittendigit-2010}, CIFAR10~\cite{krizhevsky2009learning}, SVHN~\cite{netzer2011reading}, and QNLI~\cite{wang2018glue} on common models: small convolutional neural network,  large backbone networks ResNet~\cite{he2016deep}, and the pre-trained BERT-base model~\cite{devlin2018bert}.
Experimental results illustrate that our methods can achieve higher accuracy on most models and datasets compared to \dpsgd while saving the effort of tuning $C$. 

For time cost, \dcsgde can achieve a $9\times$ speedup on hyperparameter tuning. For model performance, we consider the privacy cost of tuning and use RDP \cite{mironov2017renyi} to account for the privacy cost. Compared to other baselines\cite{abadi2016deep,du2021dynamic,andrew2021differentially}, \dcsgde achieve an improvement of $10.62\%$ on CIFAR10 using ResNet34 with $\epsilon=2$, while \dcsgdp achieve a $2.13\%$ improvement for SVHN on ResNet34 with $\epsilon=2$. This improvement becomes more pronounced compared to other methods requiring additional hyperparameter tuning.
And if we ignore the privacy cost of tuning like previous work~\cite{andrew2021differentially,du2021dynamic,yu2019differentially,tramer2020differentially}, and compare the performance of different methods \cite{abadi2016deep,du2021dynamic,andrew2021differentially} with selectively chosen hyperparameter combination, \dcsgdp and \dcsgde can still achieve a comparable or better performance compared to other baselines. 

To sum up, we make the following contributions:
\begin{itemize}
\item {We propose \dcsgd that determines a dynamic clipping threshold $C$ by examining the gradient norm distribution. We instantiate it to two mechanisms and tackle the intricate issue of adjusting the clipping threshold.} 
\item We conduct rigorous theoretical privacy analysis of \dcsgd by converting it to equivalent \dpsgd and perform a comprehensive theoretical analysis to ascertain the convergence properties of \dcsgd. The theoretical analysis illustrates the privacy and utility of our methods.
\item {Comprehensive experiments across image classification and natural language processing tasks present that both \dcsgdp and \dcsgde outperform or match the performance of the vanilla \dpsgd while reducing the time and privacy cost of extensive hyperparameter tuning.} 

\end{itemize}

\section{Background and Challenges}

\label{sec:background}
\subsection{Differential Privacy}

\begin{definition}[$ (\epsilon,\delta)$-Differential Privacy~\cite{dwork2006calibrating}]\label{dfe1}
A mechanism or algorithm $M:\mathcal{D} \to \mathbb{R}^d$ satisfies $(\epsilon,\delta)$-differential privacy if, for any pair of neighboring datasets $D, D' \in \mathcal{D}$, and any subset $S \subseteq \mathbb{R}^d$, it holds that:
\begin{equation*}
Pr[M (D) \in S] \leq e^\epsilon Pr[M (D')\in S]+\delta
\end{equation*}
\noindent where $\epsilon$ is the privacy budget, measuring the similarity of outputs on neighboring datasets, with a smaller $\epsilon$ meaning better privacy protection. $\delta$ denotes the probability that the privacy guarantee fails. 
\end{definition}

The concept of neighboring datasets varies by context, affecting the degree of DP. In example-level DP, two datasets are considered neighboring if they differ by a single example. In user-level DP, a single user may contribute multiple examples, resulting in neighboring datasets differing by several examples~\cite{mcmahan2017learning,andrew2021differentially}. 
\textit{This paper focuses on example-level DP to ensure the privacy of individual examples, and the DP sensitivity is defined as follows:}

\begin{definition}[Sensitivity of a function $f$\cite{dwork2006calibrating}]\label{def2}
Given two neighboring datasets $D,D' \in \mathcal{D}$, and a function $f:\mathcal{D}\to\ \mathbb{R}^d$. The $\mathcal{L}_m$-sensitivity of function $f$ is defined as the maximum value by which $f$ changes with a single individual,
\begin{equation*}
\Delta_{nf}=\max\limits_{D,D'}{{||f (D)-f (D')||}_m}
\end{equation*}
where ${||\cdot||}_m$ denotes the $\mathcal{L}_m$ norm. The sensitivity measures the maximum change in the query output caused by altering a single example.
\end{definition}

We leverage R\'{e}nyi DP (RDP)~\cite{mironov2017renyi} to get a tighter privacy analysis under composition~\cite{andrew2021differentially}.  
RDP is fundamentally grounded on the concept of R\'{e}nyi divergence. Given two probability distributions $P$ and $Q$ defined over $\mathbb{R}$, the Rényi divergence of order $\alpha > 1 $ is given by
\[D_\alpha (P\|Q) = \frac{1}{\alpha-1} \log \mathbb{E}_{x \sim Q} \left( \frac{P(x)}{Q(x)} \right)^\alpha,\]
where $P(x)$ denotes the density of $P$ at point $x$. In this expression, the logarithm is the natural logarithm, and the notation $x \sim Q$ signifies that $x$ is drawn from the distribution $Q$. Then, RDP is defined as follows,

\begin{definition}[$\left (\alpha,\rho\right)$-RDP~\cite{mironov2017renyi}]\label{deff6}
A mechanism $ f: \mathcal{D} \rightarrow \mathbb{R}^d $ is said to possess $ \rho $-R\'{e}nyi differential privacy of order $ \alpha $, denoted as $ (\alpha, \rho) $-RDP, if for any neighboring datasets $ D, D' \in \mathcal{D} $, the following condition is met:$
D_\alpha (f(D) || f(D')) \leq \rho 
$.
\end{definition} 
Mironov et al.~\cite{mironov2017renyi} demonstrate that RDP provides a more stringent privacy analysis under composition than $(\epsilon,\delta)$-DP. Furthermore, RDP can be translated into $(\alpha,\epsilon)$-DP using the subsequent lemma.

\begin{lemma}[From $(\alpha,\rho)$-RDP to $(\epsilon,\delta)$-DP~\cite{balle2020hypothesis}]\label{convert}
If a mechanism $\mathcal{M}$ adheres to $(\alpha,\rho)$-RDP, then it also satisfies $(\epsilon,\delta)$-DP for any $0<\delta<1$ where
$\epsilon=\rho+\log(\frac{\alpha-1}{\alpha})-\frac{\log \delta+\log \alpha}{\alpha-1}$.
\end{lemma}

Therefore, we can use RDP for privacy analysis and subsequently convert it into $(\epsilon,\delta)$-DP to achieve a more refined privacy assessment. \dpsgd can be viewed as a composition of the Subsampled Gaussian Mechanism (SGM). Mironov et al. \cite{mironov2019r} give the analysis of the SGM  under RDP.

\begin{lemma}[Privacy Analysis of SGM~\cite{mironov2019r}]
\label{SGMrdp}
Let $SG_{q,\sigma}$ be the SGM with sample rate $q$ and Gaussian noise $\mathcal{N}(0,\sigma^2S_f^2\mathbb{I})$ for some function $f$ with sensitivity $S_f$. Then, $SG_{q,\sigma}$ satisfies $(\alpha,\rho)$-RDP with the condition of 
\begin{equation*}
\small
\rho \leq D_\alpha((1-q)\mathcal{N}(0,\sigma^2)+q\mathcal{N}(1,\sigma^2)||\mathcal{N}(0,\sigma^2)).
\end{equation*}
\end{lemma}

In this study, we use Lemma~\ref{SGMrdp} to calculate the overall privacy budget. It is worth noting that the Lemma~\ref{SGMrdp} are built based on Poisson sampling. Throughout this paper, all sampling processes adhere to Poisson sampling, where each example has an equal probability of being sampled.

\subsection{DP Stochastic Gradient Descent} \label{sec2:dpsgd}

\dpsgd~\cite{abadi2016deep} is the predominant method for training neural networks with differential privacy guarantees using SGD. The core principle of DP-SGD involves introducing noise to the gradients during each iteration. To effectively add this noise, \dpsgd first bounds the sensitivity of the gradients. Due to the difficulty in predicting gradient norm distributions in deep learning, DP-SGD sets an upper bound $C$ on the ${\mathcal{L}}_2$ norm of the gradients, expressed as:

\begin{small}
\begin{equation}\label{clipfunction}
{\rm Clip}(g,C)=g\left/\max\left(1,\frac{||g||_2}{C}\right)\right.
\end{equation}
\end{small}

This method limits the influence of any single instance on 
$C$, effectively bounding the ${\mathcal{L}}_2$-sensitivity by $C$. DP-SGD introduces Gaussian noise with a standard deviation proportional to $C$ to the aggregated gradients. The algorithm then averages these gradients and updates the parameters as follows:
\begin{small}
\begin{equation}
\theta_{t+1}=\theta_{t}-\frac{\eta}{B}\left(\sum_{i\in \mathcal{B}_t}{\rm Clip}(g_{t,i},C)+\mathcal{N}(0,\sigma^2C^2\mathbb{I})\right)
\label{eq:dpsgd}
\end{equation}
\end{small}

\noindent `$B$' denotes the training batch size. The \dpsgd approach can be viewed as a composition of multiple SGMs. In practice, it is common to use Lemma~\ref{SGMrdp} for a more stringent privacy analysis under RDP, and subsequently convert this into $(\epsilon,\delta)$-DP \cite{andrew2021differentially,xiao2022differentially,yu2021not}. We adopt this strategy in our paper.

\begin{figure}[!t]
    \centering
    \includegraphics[width=0.98\linewidth]{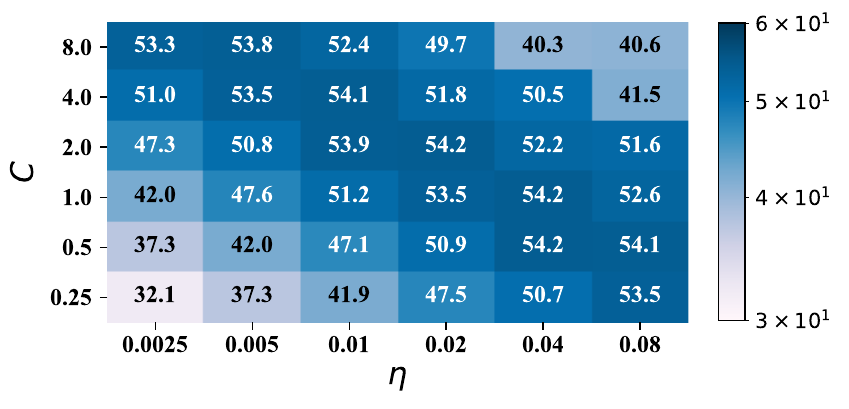}
    \vspace{-4mm}
    \caption{{Classification accuracy of ResNet18 trained on the CIFAR10 dataset, using varying clipping thresholds $C$, learning rates $\eta$, $\epsilon=8$ and SGD optimizer within the \dpsgd framework. }
    }
    \label{fig:heatmap}
\end{figure}

\subsection{Extra Privacy Cost of Hyperparameter Tuning $C$} \label{sec:privacycost}
When implementing DP-SGD to train deep learning models with differential privacy, tuning hyperparameters such as the learning rate, batch size, and clipping threshold is essential. Typically, tuning the hyperparameters entails training models multiple times with various combinations of hyperparameters, which incur extra privacy costs. 
\textit{Adjusting $C$ incurs computational overhead and increases privacy leakage risks.}
As shown in Figure~\ref{fig:heatmap}, a two-dimensional grid search for the optimal learning rate $\eta$ and $C$ (extendable to other dimensions like noise scheduling or model architecture) demands extensive model training across various hyperparameter settings, consuming significant time and computational resources. Additionally, multiple training runs for hyperparameter tuning lead to further privacy leakage~\cite{liu2019private,papernot2021hyperparameter}.

Liu et al.~\cite{liu2019private} proposed a DP hyperparameters tuning algorithm (we call it LT)  that uses random stopping with a hard limit $T$ on the number of iterations. We present the privacy analysis of LT as follows:

\begin{theorem}[Privacy Analysis of LT \cite{liu2019private}]
\label{theorem:liuhyper}
We consider any $\gamma \in [0,1]$ and $\delta_2 > 0$, and define $T$ as $\frac{1}{\gamma}\ln\frac{1}{\delta_2}$. Assuming that each iteration of \dpsgd is $(\epsilon_1, \delta_1)$-differentially private, the LT algorithm then achieves $(\epsilon', \delta')$-differential privacy. $\epsilon'$ is defined as $3\epsilon_1 + 3\sqrt{2\delta_1}$, and $\delta'$ is calculated as $3\sqrt{2\delta_1}T + \delta_2$.
\end{theorem}

Theorem~\ref{theorem:liuhyper} details the privacy cost associated with the LT algorithm. As the number of hyperparameter combinations to be tuned increases, the value of $T$ also rises. Under a constant privacy budget, this necessitates a smaller $\delta_1$. Consequently, a reduced $\delta_1$ requires a smaller $\epsilon_1$, which in turn can affect the performance of the model.

\subsection{Challenges of Dynamic Clipping}
\label{subsec:concerns}

A static $C$ may initially seem adequate but can become suboptimal as training progresses. This is because gradients typically diminish over time in standard SGD, causing a previously ideal high $C$ to result in minimal clipping and excessive noise~\cite{du2021dynamic,wei2022dpis}. However, determining an optimal clipping threshold $C$ in DP-SGD presents several significant challenges:

\noindent \textbf{Trade-offs Between Clipping and Noise.}
The selection of $C$ necessitates balancing the bias introduced by clipping against the magnitude of noise, which scales with $C$. In scenarios with constrained privacy budgets, a higher $C$ reduces bias but may degrade performance due to increased noise. A lower $C$ may induce excessive bias from over-clipping, leading to significant information loss and diminished model effectiveness.

\noindent \textbf{Lack of Private Prior Knowledge.}
The distribution of gradient norms is often private and unknown a priori, complicating the prediction of the appropriate clipping threshold. This lack of prior knowledge necessitates heuristic or empirical methods to determine $C$, which may not always yield optimal results.

\noindent \textbf{Hyperparameter Tuning Costs.}
Tuning the clipping threshold incurs significant computational overhead and additional privacy leakage (Section \ref{sec:privacycost}). Identifying the optimal $C$ often requires multiple training runs, each time-consuming and privacy-costly. This highlights the need for methods that can dynamically adjust $C$ without extensive manual tuning.

\section{\dcsgd Design}
\label{sec:DCSGD}

\subsection{Overview}
\label{motivation}

\begin{algorithm}[!t]
\begin{small}
  \caption{\dcsgd}
  \label{alg:dcsgd}
  \LinesNumbered
  \KwIn{Examples $\{ x_1,...,x_N \}$, loss function $\mathcal{L}(\theta)=\frac{1}{N}\sum_i\mathcal{L}(\theta,x_i)$. Parameters: learning rate $\eta$, total noise multiplier $\sigma$, noise multiplier for histogram $\sigma_H$, expected batch size $B$, initial clipping threshold $C_0$, initial histogram range $R_0$, count of histogram bins $b$.}
  $\textbf{Initialize:} \  \theta_0$ randomly
  $\sigma_T \get (\sigma^{-2}-\sigma_H^{-2})^{-1/2}$
  \For{$t=0$ to $T$}{
    Take a random batch $\mathcal{B}_{t}$ with sampling probability $B/N$

    $\textbf{Compute gradient}$
    
    For each $i \in \mathcal{B}_{t}$, compute $g_{t,i} \gets \nabla_{\theta_t}\mathcal{L}(\theta_t,x_i)$

     $\textbf{Compute Gradient Norms}$
    
    For each $i \in \mathcal{B}_{t}$, compute $g_{t,i} \gets ||g_{t,i}||$ and put it into $G_t$
    
    $\textbf{Private Descent}$
    
    Update parameters $\theta$ using Equation (\ref{eq:dpsgd})
    
    $\textbf{Construct Histogram}$

    $\tilde{H_t} \get \text{BuildHistogram}(G_t, R_t, b)$
    
    $\textbf{Adjust Clipping Threshold}$
    
    $C_{t+1}, R_{t+1} \get \text{Adaptive}(\tilde{H_t}, C_{t}, R_{t}, b)$
    
  }
  \KwOut{model parameters $\theta_{T+1}$}
\end{small}
\end{algorithm}

To optimize the clipping threshold $C$ while preserving privacy, we introduce a novel method that employs differentially private histograms to analyze gradient norm distributions. The key advantage of using histograms is their sensitivity of $1$. Meanwhile, the histogram can provide a comprehensive view of the whole distribution since a histogram is a direct description of the distribution.

We introduce a framework named \textbf{D}ynamic \textbf{C}lipping DP-\textbf{SGD} (\dcsgd), which dynamically adjusts the clipping threshold as outlined in Algorithm~\ref{alg:dcsgd}. Unlike traditional \dpsgd, \dcsgd incorporates two key processes: \textit{histogram construction} (Line 12) and \textit{adaptive adjusting $C$} (Line 14). \dcsgd strategically allocates the privacy budget between model training and updating $C$, as detailed in Section~\ref{subsec:sigmarelation}. During histogram construction, data points are categorized into bins, with outliers allocated to the last bin, and differential privacy is maintained through noise injection. We propose two specific strategies for adjusting the clipping threshold $C$: \dcsgd with Percentile (\dcsgdp) and \dcsgd with Expected Squared Error (\dcsgde).

\noindent \textbf{DCSGD with Percentile (\dcsgdp)}: We introduce a single hyperparameter, $p$, to dynamically adjust the clipping threshold $C$ during training (Deatailed in Section~\ref{percentilemethod}). This approach ensures that $p\%$ of the gradients remain unclipped, effectively clipping with a probability of $1-p$. Distinct from previous methods~\cite{andrew2021differentially, abadi2016deep, golatkar2022mixed}, \dcsgdp requires tuning only this additional hyperparameter and eliminates the need for an external public dataset. Moreover, it guarantees that the adjustment of $C$ adheres to DP.

\noindent \textbf{DCSGD with Expected Squared Error (\dcsgde)}:
We further developed a method that autonomously determines a clipping threshold $C$ without additional hyperparameter tuning (Detailed in Section~\ref{msemethod}).
By adaptively adjusting $C$ to minimize the expected error of individual gradients (as defined in Equation~\ref{equamse}), \dcsgde alleviates the effects of clipping and noise on individual gradients within the bounds of DP. This pioneering approach eliminates the need for hyperparameter adjustments, significantly enhancing both time efficiency and privacy protection.

The newly determined clipping threshold is applied in subsequent iterations for two primary reasons: First, as stated in Theorem~\ref{theorem:allocation}, our methodology involves the simultaneous publication of both the gradient and a one-hot histogram vector. Adjusting the current gradient with a newly adapted threshold post-histogram analysis would contravene this assumption. Second, in practical implementations, a small physical batch size is frequently used to simulate a larger logical batch size using gradient accumulation techniques. Applying the new clipping threshold within the same iteration would necessitate retaining all unclipped gradients until the logical batch size is achieved, thus restricting the utility of gradient accumulation methods.

\begin{algorithm}[!t]  
        \small
	\caption{HitogramBuild}
	\label{alg:histogrambuild}
	\LinesNumbered 
	\KwIn{Unclipped gradients' norms $\mathcal{G}=\{G_1,G_2,...,G_{|\mathcal{B}_t|}\}$, histogram range $R_t$, bin count $b$, noise multiplier $\sigma_H$}

        \textbf{Initialize} $H_t[b]=\{0,...,0\}$

	
	\For{$G_i \in \mathcal{G}$} {

            $index \gets \max (n-1,\lfloor \frac{bG_i}{R_t} \rfloor)$
	    
	    $H_t[index]\gets H_t[index] + 1$
	}
	
	
	$\tilde{H}_t\gets H_t+\mathcal{N}(0,\sigma_H^2\mathbb{I})$

	 \KwOut{Noisy histogram $\tilde{H}_t$}
\end{algorithm}

\subsection{Histogram Construction}
Accessing precise gradient distribution information is crucial for adaptively determining the clipping threshold $C$. Algorithm~\ref{alg:histogrambuild} outlines the construction of a gradient norm histogram in each training iteration under DP constraints. Data points that exceed the predefined range are allocated to the last bin. Notably, we address the inherent bias associated with the range $R$ in \dcsgdp and \dcsgde through adaptive adjustment strategies.

For DP compliance, noise is added to each histogram bin. The sensitivity of histogram publication is 1, ensuring minimal changes caused by adding or removing a single example. At each iteration, gradients and corresponding histogram information are published with Gaussian noise for privacy. Detailed privacy analysis is provided in Section~\ref{subsec:sigmarelation}. In practice, a small portion of the privacy budget is allocated to histogram publication, and $\sigma_H$ is set to minimize changes in gradient noise without incurring extra tuning costs. Additionally, a sufficiently large batch size is recommended to mitigate bias and reduce the impact of noise.

\subsection{\dcsgd with Percentile}
\label{percentilemethod}

In \dcsgd, the clipping threshold determined in the current iteration is applied to clip gradients in the subsequent iteration. Ideally, setting a percentile-based threshold should involve considering the entire training dataset, but this incurs significant time and privacy costs. In \dpsgd, where each training example is equally likely to be sampled for constructing a mini-batch. Therefore, \dcsgdp estimates the Percentile using the sampled batch and leverages the histogram of gradient norms instead of accessing the complete dataset, as presented in Algorithm~\ref{alg:hist}. \dcsgdp consists of two steps: (1) determining the clipping threshold based on the gradient norm distribution to ensure that $p\%$ of gradients remain unclipped; (2) adjusting the histogram range.

\noindent \textbf{Estimate the Percentile.} To estimate the Percentile, we calculate the midpoint of each bin as the estimated norm value. We then accumulate counts starting from the leftmost bin and continue to the right until the cumulative total exceeds $p\%$ of the overall norm count, which is determined by summing across all bins. Disregarding the added noise in the histogram, the true and estimated percentiles are within the same bin, yielding a maximum error of $\left|\frac{R}{2b}\right|$. Given the iterative nature of the $C$ adjustment, this error remains acceptable with careful selection of $R$ and $b$.

Note that $b$ and $R$ do not incur additional tuning costs.
Figure~\ref{fig:percentile_hist} depicts our method's estimation results on synthetic data. With a suitable range ($R=150$ covering most data), our approach closely approximates the true result, demonstrating low sensitivity to $\sigma_H$ and $b$. More results on true datasets can be found in Section~\ref{RQ3}.

\begin{algorithm}[!t] 
	\caption{Adaptive $C$ with Percentile}
        \small
	\label{alg:hist}
	\LinesNumbered 
	\KwIn{noisy histogram $\tilde{H}_t$, bin count $b$, current clipping threshold $C_t$, histogram range $R_t$, percentile $p$.}
        \textbf{Initialize} $S=0$
        
        $S'\get \sum_{i=0}^{b-1} \tilde{H}_t[i]$

        \For{$i \in [b]$}{
            $S+=\tilde{H}_t[i]$
            
            Use the middle of the current bin as $C_{t+1}$
            
            \textbf{If} $S\geq pS'$, \textbf{break}
        }
        
        $R_{t+1}=2C_{t+1}$
        
	 \KwOut{$C_{t+1}, R_{t+1}$}
\end{algorithm}

\begin{figure}[!t]
    \centering
    \subfloat{
        \label{fig:p_syn}\includegraphics[width=1.0\linewidth]{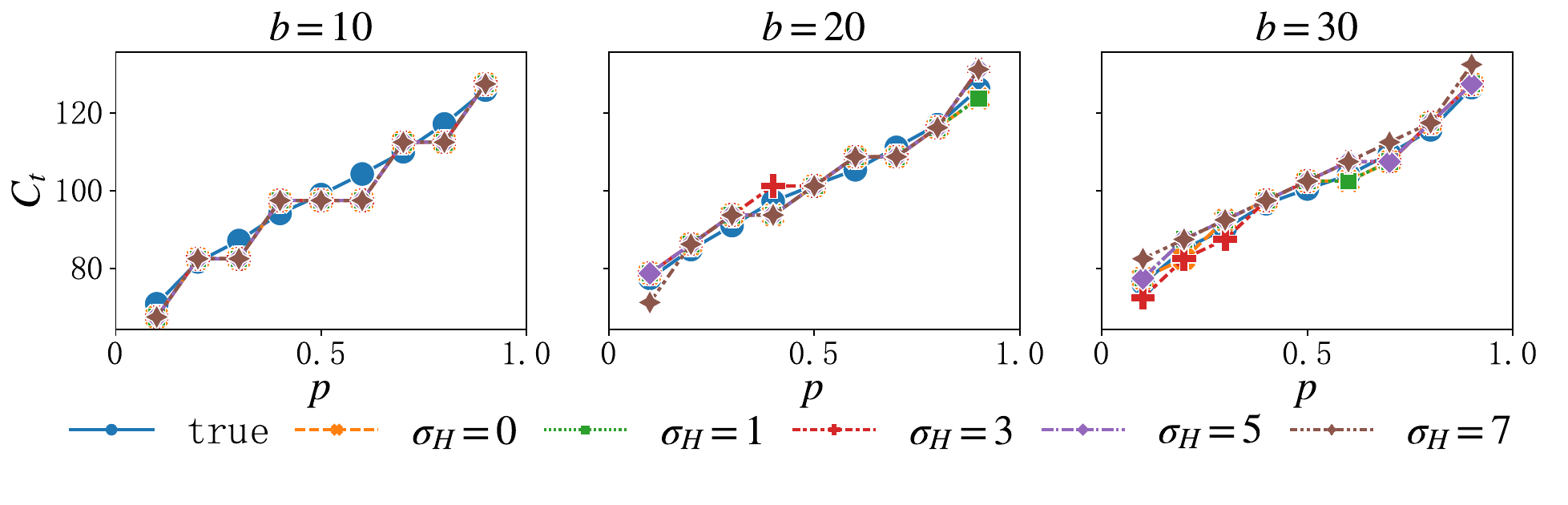}
    }
    \vspace{-2mm}
    \caption{$C_t$ according to different percentile $p$ on a set of 256 random synthetic data drawn from $\mathcal{N}(100,20^2)$ with different histogram construction and noise level. Each subfigure uses a different histogram bin count $b$ and generates data separately.}
    \label{fig:percentile_hist}
\end{figure}

\smallskip \noindent \textbf{Adjust the Range of Histogram.}
Determining the histogram's range is crucial for precisely identifying desired gradient norms. If the range is too narrow, outliers cluster in the last bin, resulting in inaccurate estimations of $C$. Conversely, an overly broad range may produce less precise results. Additionally, the variability in gradient norms across training epochs, models, and datasets complicates setting a fixed range that suits all scenarios. To this end, we propose an adaptive method that dynamically determines the histogram range based on the current clipping threshold $C$. Specifically, we configure the histogram within a $[0, 2C]$ range, where $C$ denotes the existing clipping threshold. If the histogram fails to yield a new clipping threshold corresponding to the desired Percentile, indicated when the output is the median of either the first or last bin, this output is then adopted as the new current clipping threshold, prompting an adjustment of the range for improved histogram accuracy in subsequent iterations. Given the exponential increase or decrease in the histogram range, only a minimal number of repetitions are required to rapidly determine an appropriate range. It is important to note that while this method starts with an initial value of $C$, it can be arbitrarily set as the algorithm is designed to refine it effectively.

\vspace{-0.5cm}
\subsection{\dcsgd with Expected Squared Error}
\label{msemethod}

While the method in \secref{percentilemethod} eliminates the need to tune the clipping threshold $C$, it introduces a new dataset-dependent hyperparameter, Percentile $p$. In this subsection, we devise a method that autonomously determines a clipping threshold without additional hyperparameter tuning.

\noindent {\bf Algorithm Design.}
\dpsgd updates the model parameters by $\frac{1}{B} (\sum_{i\in\mathcal{B}_t} {\rm Clip}(g_{t,i}, C_t) +  \mathcal{N}(0,\sigma_T^2C^2\mathbb{I}^d))$,  where ${\rm Clip} (g_{t,i},C_t)$ means clipping the gradient with clipping threshold $C_t$ at iteration $t$. Focusing on a single gradient $g_{t,i}$, the operation on it is equivalent to $\tilde{g}_{t,i}={\rm Clip}(g_{t,i},C)+\mathcal{N}(0,\frac{\sigma_T^2C^2\mathbb{I}}{|\mathcal{B}_t|^2})$, where $\tilde{g}_{t,i}$ is the private estimation of gradient $g_{t,i}$. {We define $\bar{g}_{t,i}={\rm Clip}(g_{t,i},C)$, and consider the randomness in noise and data sampling at iteration $t$, the expected squared error between $g_{t,i}$ and $\tilde{g}_{t,i}$ can be expressed as:}
\begin{small}
\begin{align}\label{equamse}  
    E_{t,C}&= \mathbb{E}[ (\tilde{g}_{t,i}-g_{t,i})^2] \nonumber \\
    &=\mathbb{E}[ (\tilde{g}_{t,i}-\bar{g}_{t,i})^2]+2\mathbb{E}[(\tilde{g}_{t,i}-\bar{g}_{t,i})(\bar{g}_{t,i}-g_{t,i})] \nonumber \\
    &+\mathbb{E}[(\bar{g}_{t,i}-g_{t,i})^2] \nonumber \\
    &\overset{(i)}{=}\mathbb{E}[(\tilde{g}_{t,i}-\bar{g}_{t,i})^2]+\mathbb{E}[(\bar{g}_{t,i}-g_{t,i})^2] \nonumber \\
    &\overset{(ii)}{=}\frac{\sigma_T^2C^2d}{|\mathcal{B}_t|^2}+(\frac{1}{N}\sum_{j=1}^N||g_{t,j}*\min (1,\frac{C}{||g_{t,j}||})-g_{t,j}||^2) \nonumber \\
    &=\underbrace{\frac{\sigma_T^2C^2d}{|\mathcal{B}_t|^2}}_{Variance}+\underbrace{(\frac{1}{N}\sum_{j=1}^N \max(||g_{t,j}||-C,0)^2)}_{Bias}
\end{align}
\end{small}

\noindent {(i) follows from $\tilde{g}-\bar{g}=\mathcal{N}\left(0,\frac{\sigma_T^2C^2\mathbb{I}}{|\mathcal{B}_t|^2}\right)$, indicating that this term is independent of the gradient and each of its dimensions has an expected value of 0.}
As for $(ii)$, {we note that each noise dimension is independent and the variance is $\frac{\sigma_T^2C^2}{|\mathcal{B}_t|^2}$. Given that the noise has $d$ dimensions, the total variance caused by noise is $\frac{\sigma_T^2C^2d}{|\mathcal{B}_t|^2}$.}  For the bias caused by clipping, we calculate the expected value using the average overall gradients, as each example is uniformly sampled from the entire dataset.

\begin{figure}
    \centering
    \subfloat{
        \label{fig:mse_syn}\includegraphics[width=0.95\linewidth]{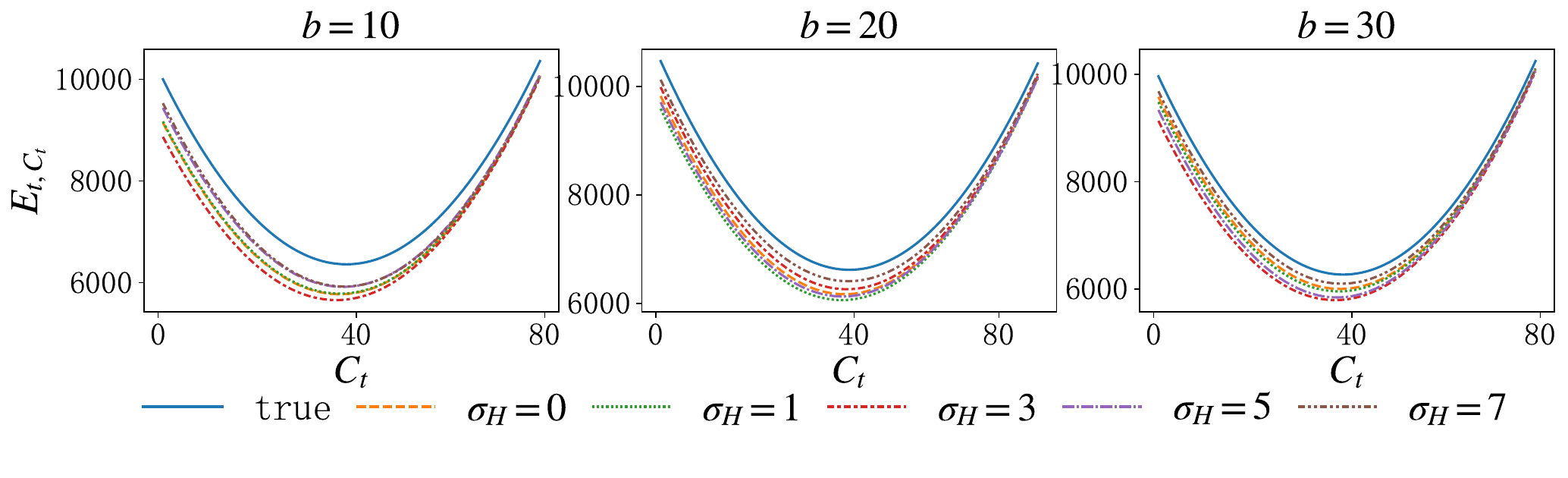}
    }
    \vspace{-2mm}
    \caption{$E_{t,C_t}$ of different $C_t$ on a set of 256 random synthetic data from $\mathcal{N}(100,20^2)$  with different histogram structures and noise levels. To compute the variance term, $\sigma_T=1, B=256, d=100000$. The histograms all have $R=120$. Each subfigure uses a different histogram bin count $b$ and generates data separately.}
    \label{fig:msehist}
\end{figure}

\begin{figure}[!t]
\vspace{-0.1cm}
    \centering
    \subfloat{
        \includegraphics[width=0.96\linewidth]{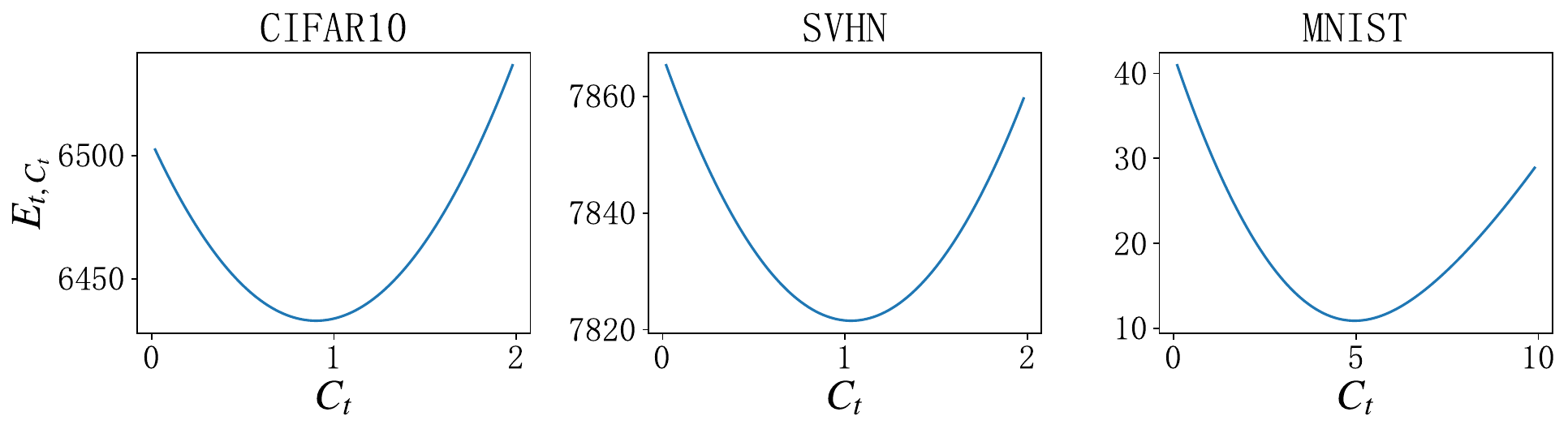}
    }
    \vspace{-2mm}
    \caption{$E_{t,C_t}$ of different $C_t$ at some iteration for CIFAR10 on ResNet18, SVHN on ResNet34, MNIST on CNN. Privacy budget $(\epsilon=8,\delta=1/|D|)$, batch size $B=256$, clipping threshold $C=1$, using default Adam optimizer.}
    \label{fig:utype}
\end{figure}

In Equation (\ref{equamse}), the larger $C$ results in stronger noise, which increases variance. However, a larger $C$ simultaneously reduces bias, as fewer gradients are clipped, and vice versa. The bias and variance go in opposite directions, and an optimal $C$ should exist. {Figure~\ref{fig:utype} illustrates the expected squared error for different clipping thresholds across various datasets on ResNet18, supporting our analysis. We can fine-tune $C$ by calculating the expected squared error and selecting the $C$ value that minimizes this error.}

{To calculate the expected squared error for various clipping thresholds $C$ while maintaining a low differential privacy cost, we have modified Equation ~
(\ref{equamse}) accordingly.} First, in \dpsgd, we apply Poisson sampling to construct the mini-batch; thereby, the mini-batch's data count differs among iterations. The actual count will leak the privacy. Moreover, the new clipping threshold at the current iteration will be utilized at the next iteration, and the count of the next iteration is unpredictable. Therefore, we use expected batch size $B$ instead of $|\mathcal{B}_t|$ in variance term. Second, to save privacy and time, we can estimate the expectation of bias terms on the mini-batch because each example in the mini-batch is sampled uniformly from the whole dataset. In summary, we get:
\begin{equation}\footnotesize
\label{equamsefinal}
\mathbb{E}[ (\tilde{g}_{t,i}-g_{t,i})^2]=\underbrace{\frac{\sigma_T^2C_t^2d}{B^2}}_{\text{Variance}}+\underbrace{(\frac{1}{|\mathcal{B}_t|}\sum_{j\in \mathcal{B}_t} \max(||g_{t,j}||-C_t,0)^2)}_{\text{Bias}}
\end{equation}

\begin{algorithm}[!t]  
        \small
	\caption{Adaptive $C$ with Minimizing Expected Squared Error}
	\label{alg:msesgd}
	\LinesNumbered 
	\KwIn{noisy histogram $\tilde{H}_t$, bin count $b$, current clipping threshold $C_t$, histogram range $R_t$.}

    $S'\get \sum_{i=0}^{b-1} \tilde{H}_t[i]$
	
    \tcp {Computing Error and Selecting $C$}

    \While{True}{

        $\mathcal{S} \gets \emptyset$

        For $i \in [1,20]$, $\mathcal{S} \gets \mathcal{S} \cup ( iC_t/10)$
	
	\For{$C' \in \mathcal{S}$} {
          Compute the bias term $B$ by using the median of each bin to estimate the norms in it.

          Compute  $E_{C'}$ with \eqref{equamsefinal}
	   
	}
        $C_{t+1}\gets {\rm arg}\underset{C'}{\min}(E_{C'})$ 

        Repeat by setting $C_t=C_{t+1}$ if $C_{t+1}$ is the boundary of candidates

    }
    \tcp {Adjusting the range of histogram}
    \If{$\tilde{H}_t[b-1]\geq 0.5S'$}{
        $R_{t+1}=2R_t$
    }
    \ElseIf{$\sum_{i=b/2}^{b-1}\tilde{H}_t[i]\leq S'/b$}{
        $R_{t+1}=0.5R_t$
    }

	 \KwOut{$C_{t+1}, R_{t+1}$}
\end{algorithm}

To determine an optimal value for $C$ with minimal expected squared error, as described in Equation (\ref{equamse}), it is necessary to obtain the norms of the gradients to calculate the bias term. 
To ensure DP, we use a histogram-based estimation (similar to what we did for \dcsgdp) of the norm distribution and subsequently perform an approximate computation of the bias term using the estimated results.
Specifically, as shown in Algorithm~\ref{alg:msesgd}, it consists of five steps: (1) get the estimation of data count at the current iteration; (2) create a candidate set of potential values for the clipping threshold $C$; (3) approximate computation of the expected squared error using the noisy histogram and selection of the optimal $C$ with the smallest expected squared error; (4) verification if the candidate set contains a value of $C$ in proximity to the optimal one. If not, the algorithm is repeated; (5) adjust the histogram range.

\smallskip \noindent \textbf{Estiamation of $E_{t,C_t}$}. {To estimate $E_{t,C_t}$, we use the midpoint of each bin in the noisy histogram as the estimated norm for that bin. Additionally, our noisy histogram designates the last bin for outliers. While these two practices might introduce a certain degree of bias in estimating the bias term, our primary focus is on the relative magnitudes of $E_{t,C_t}$ across different $C_t$ values, rather than on exact values.
Figure~\ref{fig:msehist} illustrates the variation of $E_{t,C_t}$ for a set of random synthetic data under different histogram structures and noise levels as $C$ changes. $R=120$ implies that a significant amount of data lies outside the histogram range, causing the overall estimate to be relatively smaller compared to the true value. However, regardless of the settings for $\sigma_H$ and $b$, our approximate calculation curve closely tracks the true curve. This consistency ensures that we can determine a nearly optimal clipping threshold. The above analysis indicates that our method exhibits low sensitivity to the $\sigma_H$, $b$, and $R$, thus obviating additional tuning costs.

\smallskip \noindent \textbf{Construction of Candidate Set.} Compared to \dcsgdp, the main distinction lies in constructing the candidate set for $C$. Considering the infinite possible values for $C$, we employ a strategy to build and select from a finite set of candidates. Given the characteristic `U' shape of the expected squared error curve to $C$, as depicted in Figure~\ref{fig:utype}, we implement the exponential growth strategy outlined in Algorithm~\ref{alg:hist}.
Specifically, we generate 20 candidate clipping thresholds based on the current value of $C$, computed as $0.1C, 0.2C,\cdots, 2C$. We then compute the expected error for each candidate. Notably, this computation, based on the noisy histogram, preserves privacy by adhering to the post-processing property of differential privacy. If the candidate yielding the smallest expected squared error is either $0.1C$ or $2C$, it is established as the new clipping threshold.
This action triggers repetitive execution of the selection process, ensuring the algorithm explores a broad range of thresholds and ultimately converges to a value close to the optimal $C$, situated at the lowest point of the curve.

\smallskip \noindent \textbf{Adjust the Range of Histogram.}
{We opt for a histogram range of $[0,2C]$ as we focus primarily on the norms within the specific Percentile of gradients' norms.
While our method is not overly sensitive to the exact choice of $R$, since our concern lies more with the relative magnitudes of $E_{t,C_t}$ for different $C_t$ values than with precise values, we still aim to adjust the range adaptively to accommodate the variability of gradient norms across various scenarios. For setting the range, one could either conduct experiments on a similar public dataset or employ an adaptive approach as specified in Algorithm~\ref{alg:msesgd}, which leverages the similarity in gradient distributions across neighboring training iterations. Given that \dcsgde is not highly sensitive to the histogram range $r$, as shown in Figure~\ref{fig:msehist}, we employ a ``cautious updates'' strategy to refine the range: (1) if the count in the rightmost bin exceeds $0.5|\mathcal{B}t|$, suggesting the range is too narrow for accurate estimation, we adjust $R_{t+1}$ to be double the current $R_{t}$; (2) to avoid an extensive range, if the sum of counts in the right half of the bins is less than $\frac{|\mathcal{B}t|}{n}$, we reduce $R_{t+1}$ to half of $R_{t}$.

\section{Theoretical Analysis}
\label{theoremanaly}
\subsection{Privacy Analysis}
\label{subsec:sigmarelation}
In \dcsgd, we split the total noise multiplier $\sigma$ into two parts: $\sigma_T$ for training and $\sigma_H$ for clipping threshold updating. Then, the privacy cost of \dcsgd with $\sigma_T$ and $\sigma_H$ is equivalent to vanilla \dpsgd with $\sigma$. The related theorem and proof are presented as follows:

\begin{theorem}
\label{theorem:allocation}
{\Cref{alg:dcsgd} uses Gaussian noise with a noise multiplier of $\sigma_T$ to perturb the gradient, and noise with a noise multiplier of $\sigma_H$ to perturb the histogram. This approach is equivalent to the standard \dpsgd, using a noise multiplier $\sigma$ for gradient perturbation. This equivalence holds under the calculation of privacy costs when:}
\begin{equation}
    \sigma_T = (\sigma^{-2}-\sigma_H^{-2})^{-1/2}
\end{equation}
\end{theorem}

\begin{proof}
In \dcsgd, each example, in addition to publishing its own gradient, also releases the gradient norm falling within which bin of the histogram that can be represented using one-hot encoding $H_{t,i}$. And the histogram can be constructed by $H_t=\sum_{i \in \mathcal{B}_t} H_{t,i}$. To analyze the noise added separately to the gradient and the histogram as a whole, we make a conceptual modification to the algorithm that does not influence the privacy properties like \cite{andrew2021differentially}. Instead of publishing $(g_{t,i},H_{t,i})$, example $i$ publish $(\bar{g}_{t,i}',H_{t,i}')=(\bar{g}_{t,i}/\sigma_TC,H_{t,i}/\sigma_H)$, where $\bar{g}_{t,i}$ is the gradient after clipping. Let $d$ represent the dimension of gradient, and the computation of averaged noise gradient and aggregated noise histogram can be written as:

\begin{equation}\footnotesize
    \tilde{g}=\frac{\sigma_TC}{B}(\sum_{i\in B_t}\bar{g}_{t,i}'+\mathcal{N}(0,\mathbb{I}^d)),\tilde{H}=\sigma_H(\sum_{i\in B_t}H_{t,i}'+\mathcal{N}(0,\mathbb{I}^b)) 
\end{equation}

Then, the noise addition can be rearranged as follows:
\begin{equation}
\label{equ:final}\footnotesize
    \sum_{i\in B_t}(\bar{g}_{t,i}',H_{t,i}')+\mathcal{N}(0,\sigma^2 \Delta^2 \mathbb{I}^{d+b}),\sigma \Delta=1
\end{equation}
\end{proof}

\noindent The above Equation is equal to \dpsgd with gradients in $d+b$ dimensions. $\sigma$ is the noise multiplier and $\Delta$ is the clipping threshold that is the upperbound of $||\bar{g}_{t,i}',H_{t,i}'||$. Then we handle $\Delta$:
\begin{equation}\footnotesize
\label{equ:sensitivity}
\begin{split}
    ||(\bar{\textbf{g}}_{t,i}',H_{t,i}')||&=\sqrt{\textbf{g}_{t,i,1}'^2+...+\textbf{g}_{t,i,d}'^2+H_1'^2+...+H_b'^2}\\
    &\overset{(i)}{=}\sqrt{||\bar{\textbf{g}}_{t,i}'||_2^2+H_k'^2}\\
    &\overset{(ii)}{\leq}\sqrt{(\frac{C}{\sigma_TC})^2+(\frac{1}{\sigma_H})^2} \\
    &=\sqrt{(\frac{1}{\sigma_T})^2+(\frac{1}{\sigma_H})^2} =\Delta \\
\end{split}
\end{equation}

In the above derivation, $(i)$ is due to the definition of $\mathcal{L}_2$ norm and $H_{t,i}$ is a one-hot vector so that only $H'_k$ is not 0; $(ii)$ is due to $\tilde{g}_{t,i}$ is clipped by $C$. Then we can get $\sigma_T = (\sigma^{-2}-\sigma_H^{-2})^{-1/2}$ utilizing $\sigma \Delta=1$.

For the privacy analysis of \dcsgdp and \dcsgde. We have shown the privacy analysis of \dcsgd can be converted to that of \dpsgd. Therefore, \dcsgdp is differentially private. Then, in both \dcsgdp and \dcsgde, the extra computation is based on the published noisy histogram. With the post-processing property of DP, the additional computation does not incur any privacy cost. In summary, both \dcsgdp and \dcsgde satisfy DP, and the privacy cost can be computed by converting it into an equivalent \dpsgd with \Cref{theorem:allocation}.

\subsection{Convergence Analysis}
\label{convergeanaly}

In this section, we present the convergence analysis of our proposed methods for vanilla SGD optimizer. Inspired by the work of Du et al.~\cite{du2021dynamic}, we adopt a similar convergence analysis framework. The convergence analysis results for \dcsgdp are summarized in \Cref{convergence}. The convergence analysis aims to demonstrate the impact of our dynamic clipping operation on convergence instead of directly illustrating the superiority of our methods on convergence.

\begin{theorem}[Convergence Analysis of \dcsgdp]
\label{convergence}
Let $\mathcal{L}(\theta)$ denote the loss function, and $g_t=\frac{1}{N}\sum_{i=0}^{N}g_{t,i}=\nabla\mathcal{L}(\theta_t)$ represent the actual gradient of $\mathcal{L}(\theta_t)$, where $N$ is the dataset size. The estimated gradient after clipping at iteration $t$ is represented as $\bar{g_t}=\frac{1}{B}\sum_{i\in\mathcal{B}t}{\rm Clip}(g{t,i},C)$, with $B$ being the expected batch size. Constants $L$ and $G$ are given (from the assumption that $L$-smooth and gradient norm is bounded), $d$ is the dimension of gradients, $p$ is the Percentile with $p'=1-p$, $T$ is the total number of iterations, and $q=B/N$ is the sampling rate. If we disregard the bias from histogram estimation and the privacy cost of the histogram and set $\sigma_H=\sigma$. Then, for any $\epsilon<c_1\gamma^2T, \delta>0, \sigma= c_2\frac{\q\sqrt{T\ln(1/\delta)}}{\epsilon}$, \dcsgdp satisfies the $(\epsilon,\delta)-DP$ guarantee, while \dcsgdp satisfies $(\epsilon,\delta)-DP$ guarantee, we have:

\begin{equation}\footnotesize
\label{mainDCSGDP}
\begin{split}
&\frac{1}{T}\mathbb{E}[\sum_{t=1}^T||\nabla\mathcal{L}(\theta_t)||^2] \\
&\leq\frac{1}{T\eta}(\mathbb{E}[\mathcal{L}(\theta_{1})]-\mathbb{E}[\mathcal{L}(\theta_{T+1})])+\frac{L\eta}{T}\sum_{t=1}^T\mathbb{E}[||\bar{g_t}||^2]\\
&+\underbrace{\frac{L\eta\sigma^2d}{B^2T}\sum_{i=1}^TC_t^2}_{variance}+\underbrace{\frac{1}{T}\mathbb{E}[\sum_{t=1}^T||g_t||p'G]}_{bias} \\
&\leq \mathcal{O}(\frac{1}{T\eta})+\mathcal{O}(\eta)+\mathcal{O}(\frac{T\eta d \ln(1/\delta)}{N^2\epsilon^2})+\underbrace{p'G^2}_{bias}
\end{split}
\end{equation}
\noindent By setting $T\eta=\frac{N\epsilon}{\sqrt{d\ln(1/\delta)}}$ and $\eta=\frac{\sqrt{d\ln(1/\delta)}}{N\epsilon}$, we get:
\begin{equation*}\footnotesize
\label{fin}
\begin{split}
\frac{1}{T}\mathbb{E}[\sum_{t=1}^T||\nabla\mathcal{L}(\theta_t)||^2]\leq\mathcal{O}(\frac{\sqrt{d\ln(1/\delta)}}{N\epsilon})+{\rm bias}
\end{split}
\end{equation*}
\end{theorem}

The convergence analysis of \dcsgde is similar, with a minor modification as outlined in \Cref{mainDCSGDP}:
\begin{footnotesize}
\begin{align}
\label{mainDCE}
&\frac{1}{T}\mathbb{E}[\sum_{t=1}^T||\nabla\mathcal{L}(\theta_t)||^2] \nonumber\\
&\leq \frac{1}{T\eta}(\mathbb{E}[\mathcal{L}(\theta_{1})]-\mathbb{E}[\mathcal{L}(\theta_{T+1})])+\frac{L\eta}{T}\sum_{t=1}^T\mathbb{E}[||\bar{g_t}||^2] \nonumber\\
&+\frac{1}{T}\sum_{t=1}^T(L\eta\underbrace{\frac{\sigma^2C_t^2d}{B^2}}_{Variance}+\underbrace{\frac{1}{N}\sum_{i=1}^N G\max(||g_{t,i}||-C_t,0)}_{Bias}) \nonumber\\
&\leq \mathcal{O}(\frac{1}{T\eta})+\mathcal{O}(\eta)+\mathcal{O}(\frac{T\eta d \ln(1/\delta)}{N^2\epsilon^2}))+\underbrace{G^2}_{bias}
\end{align}
\end{footnotesize}

We defer the detailed proof of the above analysis in the \textbf{Supplemental Materia}. In the above convergence analysis, since \dcsgdp and \dcsgde only differ at the setting of $C_t$ during training, they have similar analyses and finally get the same result due to our analysis's loose upper bound of convergence. The result of in \Cref{fin} is consistent with existing work\cite{du2021dynamic,chen2020understanding}. The non-vanishing is a biased term due to clipping operation, and previous work has shown \dpsgd with clipping suffers a constant regret in the worst case \cite{song2020characterizing}. Except for the bias term, the result aligns with our expectations: (1) smaller $\epsilon$ and $\delta$ mean better privacy protection and stronger noise, making it harder for the model to converge; (2) larger $d$ means a larger model, which is harder to converge.

Though the final convergence result is the same, \dcsgdp and \dcsgde have different influences on the model convergence from different angles. \dcsgdp ensures the probability that a gradient being clipped keeps constant during training, which can make the model update more stable. \dcsgde optimizes the expected squared error of a single gradient, which is similar to the variance and bias term in convergence analysis in the formulation. We believe that \dcsgde can also better balance the variance and bias in convergence analysis. Therefore, \dcsgde is more direct and reasonable.

\section{Experimental Setup}
\label{sec:exp}
\subsection{Investigated Baselines}

We compare our proposed methods against three baseline approaches: vanilla \dpsgd, the method proposed by Andrew et al.~\cite{andrew2021differentially}, and the method proposed by Du et al.~\cite{du2021dynamic}. We apply their dynamic strategy for $C$ under the RDP framework to make a fair comparison. And our methods has no limit on the privacy accounting, one can also use other advanced privacy accounting techniques like \cite{gopi2021numerical,abadi2016deep}. The approach by Andrew et al.~\cite{andrew2021differentially} adopts setting the clipping threshold to a specific percentile of gradient norms. While their algorithm is designed for federated learning, we evaluate it in the context of centralized learning. This method shares similarities with our \dcsgdp approach as they both involve the introduction of the Percentile $p$. All of these methods work on different mainstream optimizers, so in this section, we abuse the notion.
To illustrate our methods' advantage, we consider the time and privacy cost of hyperparameter tuning. In this work, we use grid search to tune the hyperparameter;  we adopt LT~\cite{liu2019private} or RDP to account for the privacy cost of hyperparameter tuning.

\subsection{Models and Datasets}
\noindent \textbf{Datasets.} We use four benchmark datasets to measure the privacy-utility trade-off: CIFAR10~\cite{krizhevsky2009learning}, MNIST~ \cite{lecun-mnisthandwrittendigit-2010}, SVHN~\cite{netzer2011reading}, and QNLI~\cite{wang2018glue}. The first three are for computer vision tasks, and the last one is for natural language inference tasks.
We defer the detailed in \textbf{Supplemental Materia}.

\noindent \textbf{Models.} In our experiments, we use a small CNN with two convolution layers (28, 058 parameters), ResNet18~\cite{he2016deep} (about 12M parameters), and ResNet34~\cite{he2016deep} (about 22M parameters) for CV tasks. They are trained from scratch. Since the batch normalization in ResNet is unsupported under DP, we replaced it with group normalization. We use a pre-trained BERT-base~\cite{devlin2018bert} for the NLP task, which is provided in the huggingface transformers repository~\cite{wolf-etal-2020-transformers}. Since the model is pre-trained, we just fine-tune the last encoder, the Bert pooler, and the classifier. Therefore, we just need to train about 8M parameters. We construct all the above models with PyTorch~\cite{paszke2019pytorch} and follow the default setting of PyTorch.

\subsection{Implementations}
\label{expsetting}
Our experiments vary the privacy budget $\epsilon \in \{1,2,4,8\}$ for different scenarios. The $\delta$ is set to $\frac{1}{|D|}$, where $|D|$ is the size of training set in all situation. We use RDP to account for the total privacy cost.
By default, all the experiments use training batch size $B=256$, and we set epochs to 20 for ResNet18 and ResNet34, 10 for small CNN, and 5 for BERT-base. We treat the train set as the private dataset and the test set as the public set, and all the performance results we report are the accuracy of the test set, omitting the $\%$ symbol.

We use two types of optimizers in this section. The first is SGD with momentum, and we set the momentum to 0.9. This one needs to tune the learning rate $\eta$ carefully. The other one is Adam~\cite{kingma2014adam}, we use the default setting of optimizer parameters $(\eta=0.001,\beta_1=0.9,\beta_2=0.999,\epsilon=10^{-8})$ since the default setting works well in previous work \cite{mohapatra2022role}.

The rest is the unique parameter of different methods. For \dpsgd, we need to tune $C$. For \dcsgdp and \dcsgde, as the analysis in \Cref{sec:DCSGD}, by default, we set $C_0=1, \sigma_H=5, b=20$ and do not need to tune them, where $\sigma_H=5$ is due to in most experiments ($\sigma$ is less than 1.2) it has least influence on the gradient noise according to \Cref{theorem:allocation}. We set $R_0=1$ for \dcsgdp and $R_0=b$ for \dcsgde. \textbf{All these parameters do not need to be tuned, and they work among all scenarios, like the default setting of Adam optimizer.} Therefore, \dcsgde does not need to tune any parameter and \dcsgdp only need to tune $p$. For the method of Andrew et al.\cite{andrew2021differentially}. We follow their setting: $\sigma_b=B/20, C_0=1$, and tune percentile $p$ and learning rate of $C$: $\eta_C$. For the method of Du et al.~\cite{du2021dynamic}. We only apply their strategy to the adaptive clipping threshold. Therefore, we need to tune $C_0, \rho_C$. 
We defer the hyperparameter grid and the candidates of different hyperparameters in the \textbf{Supplemental Material}.

\begin{table}[!t]
\centering
\setlength{\tabcolsep}{0.4mm}
\caption{Time cost of hyperparameter tuning of different methods on different model architectures.}
\label{timecost}
\resizebox{0.47\textwidth}{!}{
\begin{tabular}{p{1.4cm}|ccccc}
\toprule
Model    & \dpsgd              & \dcsgdp            & \dcsgde & Andrew~\cite{andrew2021differentially}     & Du~\cite{du2021dynamic}          \\
\midrule
CNN      & 3m6s$\times$10   & 3m35s$\times$9  & 3m34s  & 3m10s$\times$45  & 3m8s$\times$90   \\
ResNet18 & 12m15s$\times$10 & 12m30s$\times$9 & 12m46s & 12m20s$\times$45 & 12m18s$\times$90 \\
ResNet34 & 23m54s$\times$10 & 25m10s$\times$9 & 26m12s & 24m53s$\times$45 & 23m50s$\times$90 \\
BERT     & 17m30s$\times$10                   &   20m44s$\times$9               & 20m57s       &   19m58s$\times$45                 & 17m32s$\times$90 \\
\bottomrule
\end{tabular}
}
\end{table}

\begin{table}[t]
\centering
\caption{The accuracy comparison of different methods on Adam optimizer using LT~\cite{liu2019private} to account for the privacy cost of hyperparameter tuning.}
\label{LTresult}
\resizebox{0.47\textwidth}{!}{
\begin{tabular}{c|c|c|ccccc}
\toprule
Models                        & Dataset                      & $\epsilon$ & \dpsgd & \dcsgdp & \dcsgde & Andrew~\cite{andrew2021differentially} & Du~\cite{du2021dynamic}    \\ \midrule
\multirow{3}{*}{CNN}      & \multirow{3}{*}{MNIST}   & 2          & 92.39 &  92.06      & \textbf{94.18}  & 92.12  & 91.46 \\  
                          &                          & 4          & 93.84 &    93.75    & \textbf{94.70}  & 93.54  & 93.65 \\  
                          &                          & 8          & \textbf{95.29} &   94.54     &   94.95     &  94.99 & 94.86 \\ \midrule
\multirow{3}{*}{BERT}     & \multirow{3}{*}{QNLI}    & 2          & 68.68     &   67.60     & \textbf{74.08}  & 67.12  & 67.49 \\  
                          &                          & 4          & 70.29      &  69.60      & \textbf{75.56}  & 70.02  & 69.72 \\  
                          &                          & 8          &  72.27     &   71.21     & \textbf{76.42}  & 71.96  & 71.68 \\ \midrule
\multirow{6}{*}{ResNet18} & \multirow{3}{*}{CIFAR10} & 2          & 30.43 &  30.09     & \textbf{48.56}  & 29.31  & 29.43 \\  
                          &                          & 4          & 39.25 &   39.83     & \textbf{51.91}  & 38.87  & 40.08 \\  
                          &                          & 8          & 44.85 &   45.50     & \textbf{54.06}  & 45.51  & 45.19 \\ \cline{2-8} 
                          & \multirow{3}{*}{SVHN}    & 2          & 61.87 &   60.81     & \textbf{80.89}  &    57.21    & 53.72 \\  
                          &                          & 4          & 74.26 &   75.35     & \textbf{82.64}  & 74.96       & 75.56 \\  
                          &                          & 8          & 78.18 &   79.58     & \textbf{83.29}  &  79.31      & 78.45 \\ \midrule
\multirow{6}{*}{ResNet34} & \multirow{3}{*}{CIFAR10} & 2          & 24.38 &   24.12     & \textbf{43.92}  & 23.16  & 22.93 \\  
                          &                          & 4          & 33.97 &   34.18     & \textbf{47.10}  & 34.54  & 34.78 \\  
                          &                          & 8          & 39.75 &   39.68     & \textbf{50.01}  & 40.71  & 40.99 \\ \cline{2-8} 
                          & \multirow{3}{*}{SVHN}    & 2          & 49.96 &   49.01     & \textbf{79.19}  & 48.67  & 43.78 \\  
                          &                          & 4          & 70.87 &   71.29     & \textbf{81.54}  & 70.96  & 70.21 \\  
                          &                          & 8          & 74.88 &   76.22     & \textbf{82.61}  & 75.68  & 75.94 \\ \bottomrule
\end{tabular}
}
\end{table}

\section{Experiment Results and Analysis}
\subsection{Running Time}
We conduct experiments to illustrate the time cost of the training process, including hyperparameter tuning. We employ vanilla grid search for hyperparameter tuning, making the time cost equivalent to the product of the count of hyperparameter combinations and the time cost of a single training. We evaluate the time cost using the four mentioned model architectures on different training datasets: MNIST, CIFAR10, SVHN, and QNLI. The time cost of a single training is tested five times, and the average value, multiplied by the count of combinations, is considered the final result. In practical scenarios, adaptive optimizers such as Adam are commonly used to reduce time costs, as supported by previous work~\cite{mohapatra2022role}. We also use the Adam optimizer and do not tune the learning rate ($\eta$).

\Cref{timecost} presents the experimental results. Notably, both \dcsgdp and \dcsgde involve additional computations to adjust the hyperparameter $C$, resulting in slightly higher time costs per single training iteration compared to \dpsgd. A similar pattern is observed with the method proposed by Andrew et al., although its increase in time cost is comparatively smaller due to its simpler computational processes. In contrast, the approach introduced by Du et al. exhibits a time cost very close to \dpsgd, as they solely update the hyperparameter $C$ based on iteration numbers.

These methods exhibit comparable single-training time costs. The decisive factor lies in the count of potential hyperparameter combinations associated with hyperparameters that require adjustment. Both Andrew's method and Du's method introduce additional hyperparameters that necessitate tuning, leading to an exceptionally high time cost. In contrast, \dcsgde requires no hyperparameter tuning and achieves approximately a $9\times$ speedup compared to \dpsgd. While \dcsgdp requires tuning for the parameter $p$, and the quantities of $p$ and $C$ are similar, the time cost is essentially close to \dpsgd. However, $p$ is generally easier to tune than $C$ in practice, as discussed in Section~\ref{RQ3}. A similar scenario applies to the SGD optimizer, which requires tuning for the learning rate $\eta$; all five methods cannot circumvent this. So, the speedup results will be similar, and the actual time cost will increase due to the larger hyperparameter grid.

\begin{table}[!t]
\centering
\caption{The accuracy comparison of different methods on Adam optimizer using RDP to account for the privacy cost of hyperparameter tuning.}
\label{RDPresult}
\resizebox{0.47\textwidth}{!}{
\begin{tabular}{c|c|c|ccccc}
\toprule
Models                        & Dataset                      & $\epsilon$ & \dpsgd & \dcsgdp & \dcsgde & Andrew~\cite{andrew2021differentially} & Du~\cite{du2021dynamic}    \\ \midrule
\multirow{3}{*}{CNN}      & \multirow{3}{*}{MNIST}   & 2          & \textbf{94.24} &   94.18     & 94.18  & 92.37  & 91.11 \\  
                          &                          & 4          & \textbf{95.08} &   94.91     & 94.70  & 93.82  & 93.42 \\  
                          &                          & 8          & \textbf{95.99} &   95.78     & 94.95       &  94.67 & 94.06 \\ \midrule
\multirow{3}{*}{BERT}     & \multirow{3}{*}{QNLI}    & 2          & 72.76     &   72.67     & \textbf{74.08}  & 68.83  & 68.61 \\  
                          &                          & 4          & 74.28      &  73.13      & \textbf{75.56}  & 71.74  & 70.36 \\  
                          &                          & 8          &  75.03     &   75.16     & \textbf{76.42}  & 73.55  & 73.16 \\ \midrule
\multirow{6}{*}{ResNet18} & \multirow{3}{*}{CIFAR10} & 2          & 38.76 &   39.11     & \textbf{48.56}  & 31.47  & 28.44 \\  
                          &                          & 4          & 44.36 &   45.10     & \textbf{51.91}  & 38.79  & 35.94 \\  
                          &                          & 8          & 49.25 &   49.41     & \textbf{54.06}  & 45.30  & 41.79 \\ \cline{2-8} 
                          & \multirow{3}{*}{SVHN}    & 2          & 73.88 &  75.51      & \textbf{80.89}  &    65.42    & 54.91 \\  
                          &                          & 4          & 78.86 &   79.62     & \textbf{82.64}  & 74.75       & 69.65 \\  
                          &                          & 8          & 80.83 &   81.33     & \textbf{83.29}  &  79.03      & 76.05 \\ \midrule
\multirow{6}{*}{ResNet34} & \multirow{3}{*}{CIFAR10} & 2          & 33.30 &   33.75     & \textbf{43.92}  & 25.73  & 24.35 \\  
                          &                          & 4          & 40.51 &   40.65     & \textbf{47.10}  & 33.26  & 30.99 \\  
                          &                          & 8          & 43.57 &   44.38     & \textbf{50.01}  & 39.71  & 37.28 \\ \cline{2-8} 
                          & \multirow{3}{*}{SVHN}    & 2          & 69.76 &    71.89    & \textbf{79.19}  & 50.71  & 40.82 \\  
                          &                          & 4          & 75.95  &   77.09     & \textbf{81.54}  & 70.66  & 62.28 \\  
                          &                          & 8          & 78.01 &   79.75    & \textbf{82.61}  & 75.66  & 72.59 \\ \bottomrule
\end{tabular}
}

\end{table}

\begin{table*}[t]
	\centering
        \footnotesize
	\caption{The accuracy comparison of different methods for Adam optimizer without the privacy cost of tuning.}
        \label{AccComparsionAdam}
        \setlength{\tabcolsep}{4.2mm}{
        \resizebox{0.87\textwidth}{!}{
	\begin{tabular}{c|c|c|ccccc}
        \toprule
		Models&Dataset&$\epsilon$&\dpsgd&\dcsgdp&\dcsgde&Andrew~\cite{andrew2021differentially} &Du~\cite{du2021dynamic}\\
            \midrule
            \multirow{3}{*}{CNN} & \multirow{3}{*}{MNIST}&$1$&$\mathbf{94.78\pm0.21}$&$94.72\pm0.19$&$94.04\pm0.30$&$94.69\pm0.20$&$94.72\pm0.22$\\
            &&$2$&$95.62\pm0.15$&$95.08\pm0.27$&$94.19\pm0.87$&$95.45\pm0.29$&$\mathbf{95.63\pm0.24}$\\
            &&$4$&$95.95\pm0.11$&$95.60\pm0.20$&$94.65\pm0.16$&$\mathbf{96.16\pm0.15}$&$95.94\pm0.15$\\
            \midrule
            \multirow{3}{*}{BERT-base} & \multirow{3}{*}{QNLI}&$2$&$73.21\pm0.48$&$73.21\pm0.36$&$\mathbf{74.31\pm0.38}$&$73.40\pm0.49$&$74.24\pm0.36$\\
            &&$4$&$75.36\pm0.22$&$74.35\pm0.25$&$75.61\pm0.21$&$74.29\pm0.58$&$\mathbf{75.87\pm0.29}$\\
            &&$8$&$76.20\pm0.53$&$75.65\pm0.15$&$\mathbf{76.29\pm0.39}$&$75.49\pm0.15$&$76.29\pm0.54$\\
            
            \midrule
            \multirow{6}{*}{ResNet18} & \multirow{3}{*}{CIFAR10}&$2$&$48.34\pm0.45$&$48.63\pm0.42$&$48.89\pm0.48$&$49.21\pm0.23$&$\mathbf{49.48\pm0.81}$\\
            &&$4$&$51.28\pm0.16$&$51.46\pm0.15$&$51.42\pm0.38$&$51.49\pm0.26$&$\mathbf{52.28\pm0.64}$\\
            &&$8$&$53.81\pm0.46$&$54.18\pm0.32$&$53.76\pm0.27$&$54.30\pm0.78$&$\mathbf{54.42\pm0.32}$\\
            
            \cline{2-8}
            
            &\multirow{3}{*}{SVHN}&$2$&$80.05\pm0.21$&$80.42\pm0.28$&$\mathbf{81.39 \pm 0.46}$&$80.96\pm0.08$&$81.07\pm0.29$\\
            &&$4$&$81.03\pm0.46$&$81.43\pm0.22$&$\mathbf{82.45\pm0.36}$&$81.50\pm0.29$&$82.15\pm0.22$\\
            &&$8$&$81.92\pm0.46$&$83.06\pm0.13$&$\mathbf{83.44\pm0.40}$&$83.20\pm0.23$&$83.07\pm0.26$\\

            \midrule
            \multirow{6}{*}{ResNet34}
            &\multirow{3}{*}{CIFAR10}&$2$&$43.72\pm0.40$&$43.66\pm0.38$&$44.07\pm0.59$&$44.15\pm0.28$&$\mathbf{44.52\pm0.60}$\\
            &&$4$&$46.89 \pm 0.59$&$47.34\pm0.26$&$47.30\pm0.44$&$\mathbf{48.25\pm0.42}$&$47.75\pm0.46$\\
            &&$8$&$49.79\pm 0.36$&$49.28\pm0.44$&$50.10\pm0.28$&$50.13\pm0.24$&$\mathbf{50.26\pm0.53}$\\
            
            \cline{2-8}
            
            &\multirow{3}{*}{SVHN}&$2$&$78.47\pm0.33$&$78.90\pm0.49$&$\mathbf{79.26 \pm 0.28}$&$79.24\pm0.31$&$79.06\pm0.40$\\
            &&$4$&$79.70\pm0.47$&$80.08\pm0.18$&$\mathbf{81.67\pm0.34}$&$80.53\pm0.33$&$80.37\pm0.28$\\
            &&$8$&$80.66\pm0.30$&$81.81\pm0.39$&$\mathbf{82.47\pm0.46}$&$82.12\pm0.18$&$81.60\pm0.30$\\
        \bottomrule
	\end{tabular}
	}
 }
\end{table*}

We emphasize that the analysis above specifically applies to our experimental setup, given that the number of candidate hyperparameters, a user decision greatly influences the time cost. Nevertheless, our method consistently reduces the number of hyperparameters requiring tuning, providing a clear advantage in terms of hyperparameter tuning time costs.

\subsection{Privacy-Performance Trade-off}
In this subsection, we analyze the utility of our methods, assessing the performance of the final outputted DP model from two perspectives: one considering the privacy cost of hyperparameter tuning and the other not.

\noindent \textbf{Considering the Privacy Cost of Tuning.}
Previous studies have demonstrated that hyperparameter tuning introduces additional privacy costs~\cite{abadi2016deep,liu2019private,papernot2021hyperparameter}. In this subsection, we employ two methods to quantify the overall privacy cost. The first method involves using the LT algorithm~\cite{liu2019private}, and the second method employs grid search, treating hyperparameter tuning as a sequential composition and applying Renyi Differential Privacy (RDP). According to previous research~\cite{mohapatra2022role}, these two methods are suitable for different scenarios, with LT performing better in cases involving a large number of training executions, while RDP performs better in the opposite scenario.
Our experiment compares the performance of the best model obtained under the same predefined total privacy budget. For MNIST, where $\epsilon=1$ is deemed too small, we use $\epsilon=8$ instead. Additionally, due to the potentially large values of $\sigma$ in \dcsgdp, we adjust the $\sigma_H$ accordingly (setting it to 8 for $2\leq \sigma \leq 3$, and 12 for $\sigma > 3$) to ensure minimal impact on $\sigma$, thus incurring no hyperparameter tuning cost.

For the privacy analysis of LT in \Cref{theorem:liuhyper}, following~\cite{mohapatra2022role}, we set an extremely small $\delta_2$ ($10^{-20}$) and $\gamma=1/2G$, where $G$ represents the count of possible hyperparameter combinations. This choice implies that the expected running time of LT is $2G$, ensuring that we can practically explore almost all hyperparameter combinations. Under these conditions, we can derive the privacy cost for a single training. In the case of RDP, where the grid search involves repeatedly training $G$ models, the privacy cost for a single training can be obtained through $G$-composition.

\begin{table*}[t]
	\centering
        \footnotesize
        \vspace{1mm}
        \setlength{\textfloatsep}{-10pt}
\setlength{\tabcolsep}{4mm}{
	\caption{The accuracy comparison of different methods for SGD optimizer without the privacy cost of tuning.}
        \label{AccComparsion}
        \resizebox{0.87\textwidth}{!}{
	\begin{tabular}{c|c|c|ccccc}
        \toprule
		Models&Dataset&$\epsilon$&\dpsgd&\dcsgdp&\dcsgde&Andrew~\cite{andrew2021differentially}&Du~\cite{du2021dynamic}\\
            \midrule
            \multirow{3}{*}{CNN} & \multirow{3}{*}{MNIST}&$1$&$94.51\pm 0.19$&$93.98\pm 0.25 $&$94.06 \pm 0.35$&$94.16\pm0.17$&$\mathbf{95.23\pm0.30}$\\
            &&$2$&$95.86 \pm 0.08$&$94.86\pm0.15$&$94.34\pm 0.14$&$95.53\pm0.24$&$\mathbf{96.10\pm 0.41}$\\
            &&$4$&$96.25\pm 0.20$&$95.89\pm0.18$&$94.64\pm 0.14$&$96.34\pm0.07$&${96.56\pm0.21}$\\
            \midrule
            \multirow{3}{*}{BERT-base} & \multirow{3}{*}{QNLI}&$2$&$73.48\pm0.22$&$ 73.13 \pm 0.40 $&$74.17\pm 0.84$&$73.68\pm 0.80$&$\mathbf{74.20\pm0.48}$\\
            &&$4$&$75.29\pm0.69$&$74.94 \pm 0.59$&$75.33 \pm 0.89$&$74.39\pm0.88$&$\mathbf{75.58\pm0.74}$\\
            &&$8$&$76.70 \pm 0.73$&$75.81 \pm 0.54$&$\mathbf{76.74 \pm 0.71}$&$76.12\pm0.48$&$76.65\pm0.82$\\
            
            \midrule
            \multirow{6}{*}{ResNet18} 
            &\multirow{3}{*}{CIFAR10}&$2$&$48.88\pm0.60$&$49.01 \pm 0.27$&$48.05\pm0.11$&$48.98\pm0.68$&$\mathbf{50.08 \pm 0.23}$\\
            &&$4$&$51.19\pm0.50$&$52.16\pm0.55$&$51.62\pm0.43$&$52.35\pm0.10$&$\mathbf{52.98 \pm 0.30}$\\
            &&$8$&$54.13\pm0.27$&$53.81\pm0.33$&$53.91\pm 0.25$&$53.83\pm 0.21$&$\mathbf{54.74 \pm 0.29}$\\
            
            \cline{2-8}
            
            &\multirow{3}{*}{SVHN}&$2$&$80.47\pm0.41$&$80.74\pm0.15$&$\mathbf{81.76\pm0.24}$&$80.69\pm0.32$&$81.48\pm 0.29$\\
            &&$4$&$81.80\pm0.13$&$81.80\pm0.16$&$\mathbf{82.77\pm0.30}$&$81.96\pm0.15$&$82.52\pm 0.06$\\
            &&$8$&$82.48\pm0.27$&$82.23\pm0.10$&$\mathbf{83.45\pm0.22}$&$82.76\pm0.06$&$82.90\pm 0.17$\\

            \midrule
            \multirow{6}{*}{ResNet34} 
            &\multirow{3}{*}{CIFAR10}&$2$&$44.24\pm0.35$&$45.41\pm0.76$&$44.00\pm0.49$&$45.49\pm0.31$&$\mathbf{45.65\pm0.40}$\\
            &&$4$&$47.43\pm0.33$&$47.80\pm0.29$&$47.48\pm0.33$&$47.75\pm0.45$&$\mathbf{49.49\pm0.42}$\\
            &&$8$&$49.60\pm0.36$&$50.70\pm0.65$&$49.73\pm0.51$&$51.08\pm0.79$&$\mathbf{51.36\pm0.38}$\\
            
            \cline{2-8}
            
            &\multirow{3}{*}{SVHN}&$2$&$78.38\pm0.11$&$78.60\pm0.53$&$79.70\pm0.75$&$79.49\pm0.58$&$\mathbf{79.88\pm0.25}$\\
            &&$4$&$79.99\pm0.48$&$80.72\pm0.35$&$\mathbf{81.73\pm0.20}$&$80.91\pm0.25$&$81.40\pm0.27$\\
            &&$8$&$80.70\pm0.47$&$81.12\pm0.41$&$\mathbf{82.39\pm0.22}$&$81.59\pm0.21$&$82.05\pm0.41$\\
        \bottomrule
	\end{tabular}
	}}
 \vspace{-0.4cm}
\end{table*}

When considering the privacy cost of hyperparameter tuning, one will choose an optimizer that needs less hyperparameter to ensure the model performance. Therefore, we conduct experiments on the Adam optimizer as Mohapatra et al. \cite{mohapatra2022role} suggested \Cref{LTresult} and \Cref{RDPresult} show the result using LT and RDP separately. \dcsgde does not need to tune any hyperparameter, and it will only train the model once with the total privacy cost. Therefore, it has the same performance in both situations. 

In the results obtained using LT, \dcsgde demonstrates strong performance on CIFAR10 and SVHN, attributed to the complexity of these datasets and their sensitivity to noise levels. However, its improvement on the MNIST dataset is not as significant, as MNIST and CNN are relatively simple, and the noise level has a lesser impact on the final model performance. Moreover, a smaller $\epsilon$ leads to a more pronounced improvement for \dcsgde compared to other methods, as a small total privacy budget causes a more significant difference in the noise multiplier among different methods.
On the other hand, for the remaining methods, LT results in an inflation of $\epsilon$ by at least $3\times$. The other four methods exhibit similar performance, as LT is not highly sensitive to the count of hyperparameter combinations. Nevertheless, it's crucial to note that \dpsgd and \dcsgdp have significantly lower time costs compared to the methods of Andrew et al. and Du et al. In summary, with LT, \dcsgde performs well across various scenarios, especially on complex datasets with a small privacy budget, and \dcsgdp achieves comparable or superior performance compared to other methods.

In the results obtained using RDP, \dcsgde continues to exhibit the best performance, and the analysis for \dcsgde is consistent with previous observations. However, the differences in performance among the remaining methods become more pronounced as the total privacy cost is related to the count of hyperparameter combinations. Overall, the ranking in terms of performance is \dcsgdp$<$\dpsgd$<$Andrew$<$Du, aligning with the counts of hyperparameter combinations for each method.

In summary, regardless of the method used to measure the cost of hyperparameter tuning, both \dcsgdp and \dcsgde benefit from a smaller scale of hyperparameter tuning, leading to certain performance improvements.

\noindent \textbf{Without Considering the Privacy Cost of Tuning.}
\label{accNoHyper}
Most existing work ignores the privacy cost of hyperparameter tuning \cite{tramer2020differentially,wei2022dpis,andrew2021differentially,yu2021differentially}. Therefore, we also conduct experiments under this situation to better compare our method with other existing approaches. Since we do not consider the privacy cost, we apply both Adam and SGD optimizers in this subsection. Additionally, we can apply a small privacy budget for MNIST.

\Cref{AccComparsionAdam} and \Cref{AccComparsion} illustrate the experiment results separately. We test different hyperparameter combinations, identify the best one, repeat the training 5 times, and report the average accuracy. From the results, we draw the following analysis: (1) In general, both \dcsgdp and \dcsgde perform well with both optimizers, achieving performance close to or better than other existing algorithms. Moreover, our methods require fewer hyperparameter adjustments and are more time-efficient. (2) \dcsgdp shows performance close to the method of Andrew et al. This is because \dcsgdp and \dcsgde both adopt the concept of Percentile, but \dcsgdp tunes fewer hyperparameters. (3) The method of Du et al. can achieve the best performance under most settings, but it comes with the cost of tuning the most hyperparameters. (4) \dcsgde does not perform as well on MNIST; we attribute this to MNIST and CNN being simple enough to tolerate stronger noise, rendering the variance term in the expectation-squared error less sensitive to the impact of noise.

Under the situation of not considering the privacy cost of tuning, one can still choose \dcsgde and \dcsgdp, as they achieve comparable or better performance compared to other methods with less tuning time cost.

\begin{figure}[tp]
    \centering
    \subfloat{
        \label{diffp:legend}\includegraphics[width=0.98\linewidth]{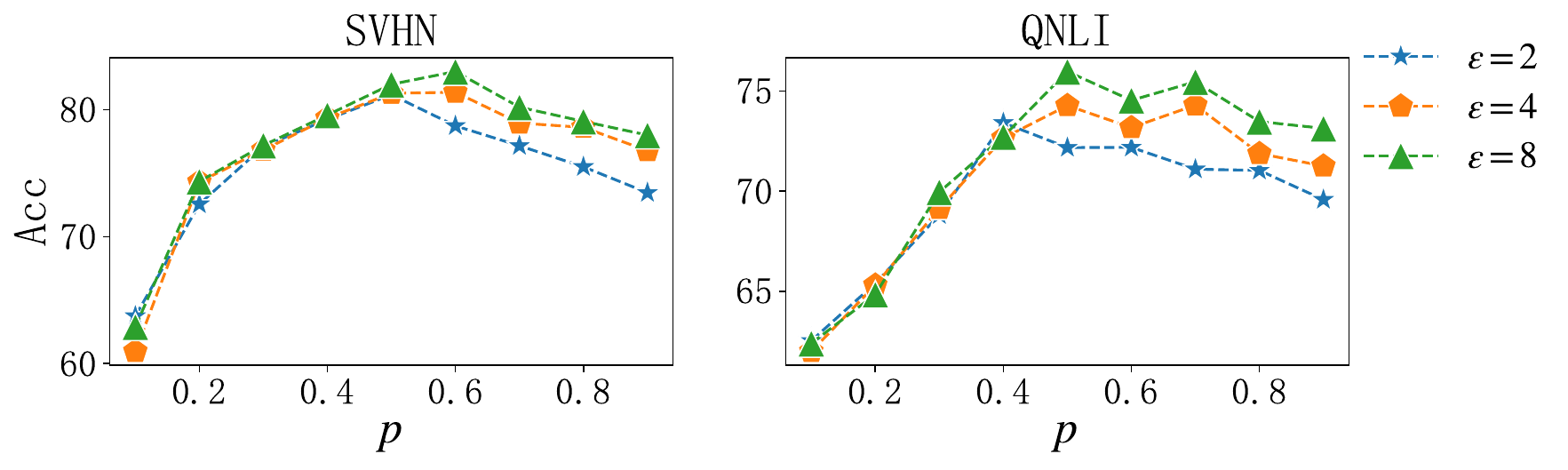}
    }
    \vspace{-0.2cm}
    \caption{The accuracy of different $p$ for \dcsgdp on SVHN and QNLI.}
    \vspace{-0.5cm}
    \label{fig:diffp}
\end{figure}

\begin{figure}[tp]
    \centering
    \subfloat{        \label{diffb:svhn}\includegraphics[width=\linewidth]{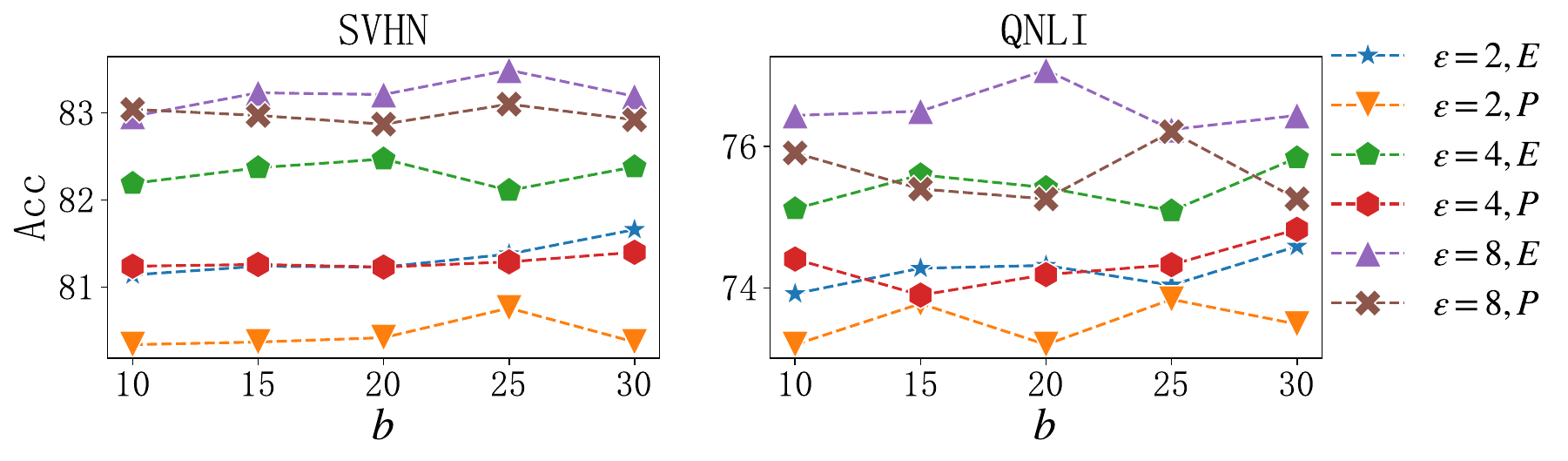}
    }
    \vspace{-0.1cm}
    \caption{The accuracy of $b$ for \dcsgdp (P) and \dcsgde (E) on SVHN and QNLI.}
    \label{fig:diffb}
\end{figure}

\begin{figure}[tp]
    \centering\subfloat{\label{diffsigma:svhn}\includegraphics[width=\linewidth]{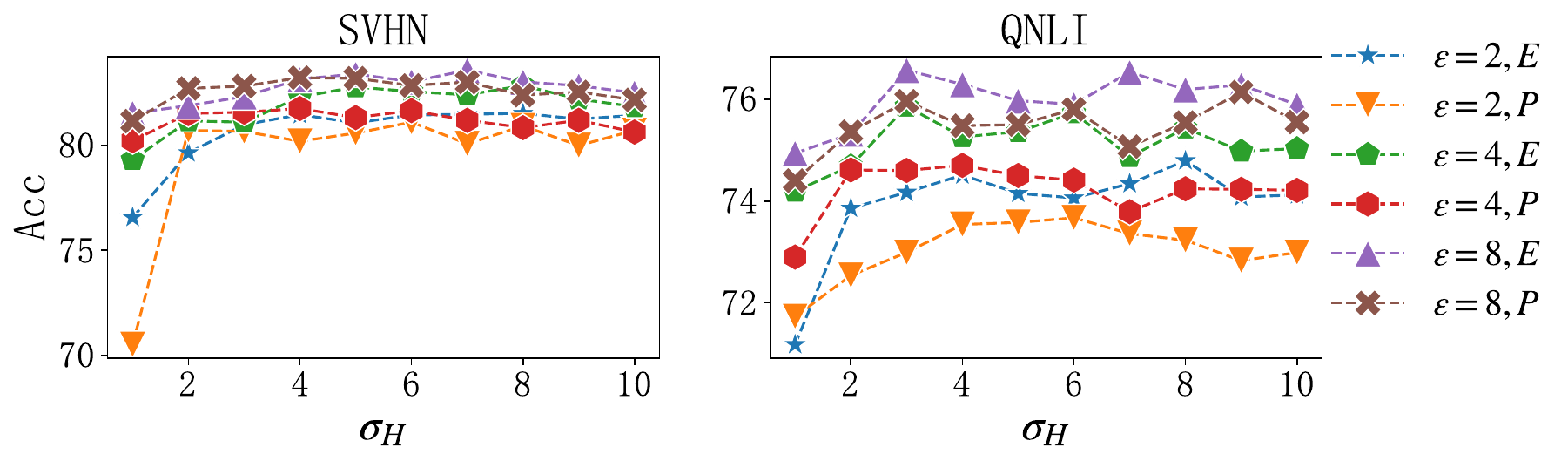}
    }
    \vspace{-0.1cm}
    \caption{The accuracy of different $\sigma_H$ for \dcsgdp (P) and \dcsgde (E) on SVHN and QNLI.}
    \label{fig:diffSigma}
    \vspace{-0.3cm}
\end{figure}

\subsection{The Influence of Parameters}
\label{RQ3}

To analyze the influence of $p, b, \sigma_H$ in \dcsgdp and \dcsgde separately and illustrate that we do not need to tune $b, \sigma_H$, we conduct experiments on SVHN (ResNet18) and QNLI (BERT-base) with the Adam optimizer. More experiments on other datasets can be found in \textbf{Supplemental Materia}. We use the best $p$ of different scenarios from \Cref{accNoHyper} as the default value for \dcsgdp in this subsection. When evaluating the impact of one parameter, the other parameters are fixed at their default values.

\Cref{fig:diffp} illustrates the influence of percentile $p$ on accuracy. $p$ significantly impacts the final model performance and needs to be chosen carefully. The best $p$ differs when $\epsilon$ and the dataset change, but the best $p$ is close among different privacy budgets. It seems that the more complex the task is, the smaller $p$ we should choose. Besides, the best hyperparameter combination from \Cref{accNoHyper} shows that the best $p$ remains close for a dataset on different models (e.g., SVHN on ResNet18 and ResNet34). All the above indicates that the choice of $p$ follows a pattern.

\Cref{fig:diffb} shows the accuracy of different $b$.The choice of $b$ has less impact on the model accuracy, considering the randomness during training. Therefore, we do not need to tune $b$. \Cref{fig:diffSigma} shows the variation in accuracy under different $\sigma_H$. We can find that the model performance is not sensitive to $\sigma_H$ except for the situation where $\sigma_H$ is small and it incurs a significant increase in the noise multiplier of the gradient, e.g., $\sigma_H=1, \epsilon=2$ in this subsection. Therefore, one only needs to choose $\sigma_H$ by ensuring that it won't cause a large increase in $\sigma_T$ and it does not incur extra tuning costs.

\section{Related Work}
\label{sec:rw}

\noindent \textbf{DP-SGD and Variants.}
Abadi et al.~\cite{abadi2016deep} introduced DP-SGD, a foundational method for training deep learning models with DP. They suggested using the median of unclipped gradient norms as the clipping threshold but lacked a method with formal privacy guarantees. For privacy analysis, they used Moments Accountant (MA), later improved by Mironov\cite{mironov2019r} with Rényi Differential Privacy (RDP). Efforts to enhance DP-SGD include Papernot et al.~\cite{papernot2021tempered}, who replaced ReLU with tempered sigmoids for better performance, and Yu et al.~\cite{yu2021large, yu2021not}, who reduced gradient dimensionality to lower noise in large models. Yu et al.~\cite{yu2019differentially} also introduced a dynamic privacy budget allocator under Concentrated Differential Privacy (CDP).
Xiao et al.~\cite{xiao2022differentially} proposed ModelMix, using random aggregation of intermediate model states for updates. Liu et al.~\cite{decompositionDPSGD} decomposed gradients into orthogonal and parallel components, focusing the privacy budget on the orthogonal part. Sha et al.~\cite{heavytails} optimized DP-SGD under a heavy-tailed distribution, applying different clipping thresholds for the heavy tail and light body of gradients.
Despite these improvements, all methods still rely on clipping to bound gradient sensitivity.

\noindent \textbf{Adaptive Clipping.}
Recent studies have tackled the challenges associated with the clipping threshold by focusing on its adaptation. Pichapati et al.~\cite{pichapati2019adaclip} introduced a coordinate-wise adaptive clipping method to reduce gradient noise. Andrew et al.~\cite{andrew2021differentially} applied percentile-based updates in federated learning with DP-FedAvg, using a geometric update scheme and a hyperparameter to control update speed. Golatkar et al.~\cite{golatkar2022mixed} used a public dataset for adaptive adjustment of $C$ through percentiles, avoiding additional privacy costs, though finding a suitable public dataset remains challenging. Du et al.~\cite{du2021dynamic} dynamically adjusted both the clipping threshold and noise power based on iteration number under Gaussian Differential Privacy. However, their method requires an initial $C$ tailored to specific settings due to the lack of norm distribution data and involves hyperparameters that require tuning. In contrast, our DC-SGD method directly accesses the norm distribution under DP to determine the optimal $C$, eliminating the need for a public dataset and reducing dependency on hyperparameters.

\noindent \textbf{Normalization-Based Clipping.}
The second approach eliminates the need for a specific clipping threshold. De et al.~\cite{de2022unlocking} proposed normalizing the gradient using $C$, effectively incorporating it into the learning rate $\eta$. They demonstrated that using a small $C$ (e.g., $C=1$) consistently yields strong results, allowing the focus to shift to tuning 
$\eta$. Similarly, Bu et al.~\cite{bu2022automatic} and Yang et al.~\cite{yang2022normalized} introduced normalization-based clipping functions that combine $\eta$ and $C$, thereby eliminating the need for a separate clipping threshold. These functions normalize the influence of each sample to $1$, ensuring bounded sensitivity.
While these methods alter the traditional clipping strategy within DP-SGD, our work retains the standard approach, focusing on dynamically adjusting $C$. Although these approaches are orthogonal to ours, we conducted comparative experiments with AutoClip, proposed by Bu et al.~\cite{bu2022automatic}, with results detailed in the Supplemental Material.

Amin et al.\cite{amin2019bounding}, Chen et al.\cite{chen2020understanding}, and Zhang et al.\cite{zhang2022understanding} explored the impact of clipping on convergence. Liu et al.\cite{liu2019private}, Mohapatra et al.\cite{mohapatra2022role}, and Papernot et al.\cite{papernot2021hyperparameter} underscored the privacy costs of hyperparameter tuning and proposed mitigation strategies.
Our work focuses on minimizing the time and privacy costs of hyperparameter tuning. DC-SGD reduces privacy costs during training while requiring fewer hyperparameter combinations while maintaining strong model performance. Recently, Liu et al. \cite{liu2024differentially} and Zhang et al. \cite{zhang2024dpzero} combined DP-SGD and Zeroth-Order Optimization to tackle the challenges of fine-tuning LLM with DP: high time and memory cost and bad model utility. And their works still need a clipping threshold to clip the gradient. We think we can combine their works with our methods since hyperparameter tuning is very expensive in the LLM scenario. We leave this as our future research direction.

\section{Conclusion}
\label{conclusion}

This paper introduces DC-SGD, a novel differentially private training solution with an adaptive dynamic clipping threshold. Leveraging the gradient norm distribution, we propose two mechanisms for dynamically setting the clipping threshold: DC-SGD-P uses a percentile 
$p$, while DC-SGD-E optimizes the expected squared error. We formally establish privacy and convergence guarantees for DC-SGD and validate its effectiveness through extensive experiments on benchmark datasets across various deep models. The results demonstrate that DC-SGD outperforms or matches DP-SGD while reducing hyperparameter tuning costs. Additionally, DC-SGD integrates seamlessly with both SGD and Adam optimizers.

\section{Acknowledgment}
We sincerely thank all the anonymous reviewers for their valuable feedback and constructive suggestions.
This research received partial support from the National Natural Science Foundation of China under No.62302441. This work was partly supported by Key Research and Development Plans of Guangxi Province (Granted No.AB22080077).
The research was also supported by the Information Technology Center of Zhejiang University and the Supercomputing Center of Hangzhou City University. Wenzhi Chen is the corresponding author.


\bibliographystyle{IEEEtran}
\bibliography{bib}

\vspace{3mm}
\begin{center}
    \textbf{\Large{Supplemental Material}}
\end{center}

\appendices

\setcounter{section}{0}
\renewcommand\thesection{\Alph{section}}

\maketitle

\section{The change of the average norm}

\begin{figure}[htp]
    \centering
    \includegraphics[width=0.8\linewidth]{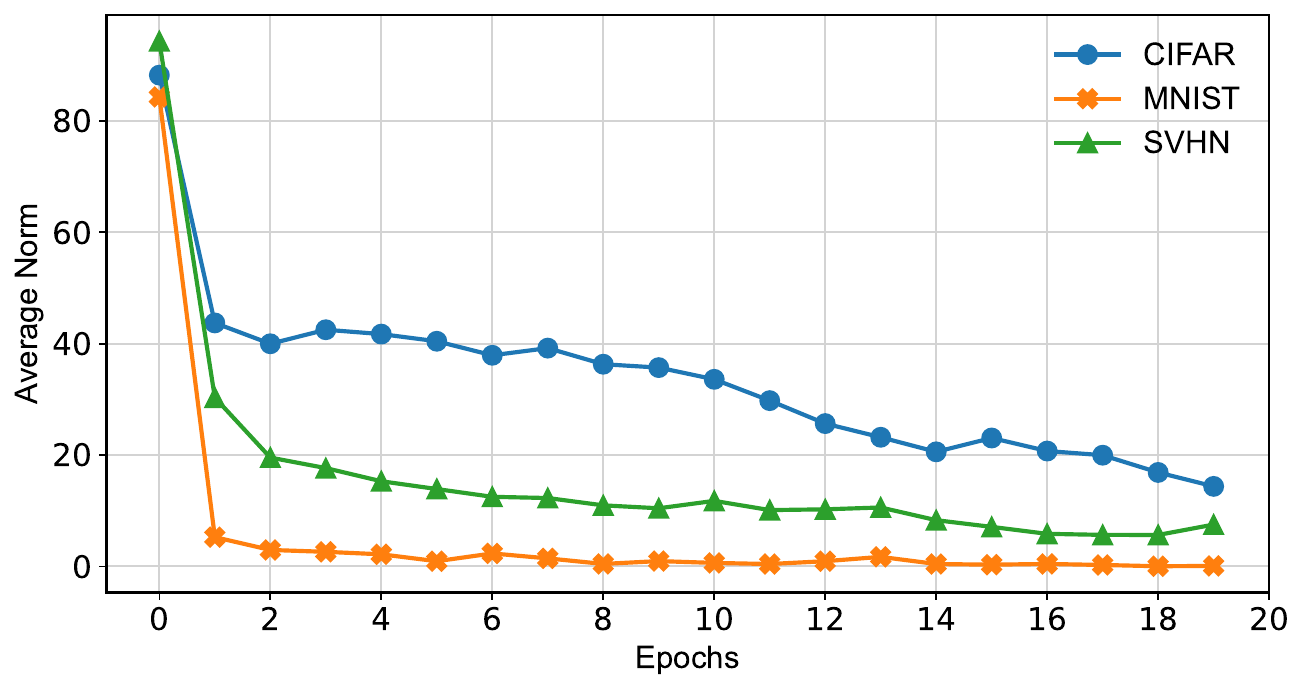}
    \caption{The change of the average norm during the training process using the standard SGD algorithm. All utilize the ResNet18 model for training. Each point is computed from $2000$ randomly selected samples. 
    }
    \label{fig:norm_change}
\end{figure}

Figure~\ref{fig:norm_change} shows the change of the average norm during the training process using the standard SGD algorithm on the ResNet18 model. Each point is computed from $2000$ randomly selected samples. 
We can see that gradients typically diminish over time in standard SGD, causing a previously ideal high $C$ to result in minimal clipping and excessive noise~\cite{du2021dynamic,wei2022dpis}.
A static $C$ may initially seem adequate but can become suboptimal as training progresses.

\section{Convergence Analysis of \dcsgdp}

\label{proofP}
To analyze the convergence of \dcsgdp with the SGD optimizer, we borrow the proof idea in \cite{du2021dynamic}. First, we need the assumptions as follows:

\begin{assumption}[$L$-Smooth]
\label{asumsmooth}
Let $\mathcal{L}(\theta)$ represent the objective function, $g_t$ denotes the gradient of $\mathcal{L}(\theta_t)$. Then $\forall \theta_t, \theta_{t+1}$, there exists an non-negative constant $L$ such that:

\begin{equation}
    \mathcal{L}(\theta_{t+1})-\mathcal{L}(\theta_{t})\leq \langle \nabla\mathcal{L}(\theta_t),(\theta_{t+1}-\theta_t)\rangle+\frac{L}{2}||\theta_{t+1}-\theta_t||^2
\end{equation}

\end{assumption}

\begin{assumption}[Bound of Gradient]
\label{boundofg}
Let $\mathcal{L}(\theta)$ represent the objective function, $g_{t,i}$ denotes the gradient of $\mathcal{L}(\theta)$ on input $x_i$. Then $\forall i,\theta_t$, $||g_{t,i}||\leq G$.
\end{assumption} 
This assumption is common used in previous work\cite{du2021dynamic,chen2020understanding}.

\begin{assumption}[Lower Bound of $\mathcal{L}$]
\label{boundofl}
For all $\theta_t$ and and some constant $\mathcal{L}_*$, we have $\mathcal{L}(\theta_t)>\mathcal{L}_*$.
\end{assumption}

To analyze the relation of convergence and privacy budget, we use the privacy analysis of \dpsgd with MA~\cite{abadi2016deep} as follows:

\begin{theorem}[Privacy Analysis of \dpsgd using MA~\cite{abadi2016deep}]
\label{abadiMA}
There are constants $c_1$ and $c_2$ such that, given a sampling probability $q=B/N$ and the number of steps $T$, \dpsgd achieves $(\epsilon,\delta)$-differential privacy for any $\delta > 0$. This holds true for any $\epsilon < c_1q^2T$, provided that
\begin{equation}\small
    \sigma\geq c_2\frac{q\sqrt{T\ln(1/\delta)}}{\epsilon}
\end{equation}
\end{theorem}

In this section, we use $\mathcal{L}(\theta)$ to represent the loss function, use $g_t=\frac{1}{N}\sum_{i=0}^{N}g_{t,i}=\nabla\mathcal{L}(\theta_t)$ to represent the actual gradient of $\mathcal{L}(\theta_t)$, where $N$ is the size of the dataset and $g_{t,i}$ is gradient of data $x_i$ at itration t. And we use $\bar{g_t}=\frac{1}{B}\sum_{i\in\mathcal{B}_t}{\rm Clip}(g_{t,i},C_t)$ to represent the estimated gradient after clipping, where $B$ is the batch size. And use $\tilde{g_t}=\frac{1}{B}\sum_{i\in\mathcal{B}_t}{\rm Clip}(g_{t,i},C_t)+\mathcal{N}(0,\sigma_T^2C_t^2\mathbb{I}^d)$ to represent the estimated gradient after clipping and noise adding. Since we use a dynamic $C$, we use $C_t$ to represent $C$ at iteration $t$. Besides, we focus on the influence of the adaptive adjustment of $C$ on the convergence of the model, we do not consider the bias caused by histogram estimation and ignore the privacy cost of the histogram; simply set $\sigma_T=\sigma$. Therefore, $\tilde{g_t}=\frac{1}{B}\sum_{i\in\mathcal{B}_t}{\rm Clip}(g_{t,i},C_t)+\mathcal{N}(0,\sigma^2C_t^2\mathbb{I}^d)$

The update of \dcsgdp of iteration $t$ can be written as follows:

\begin{equation}
    \theta_{t+1}=\theta_t-\eta_t\tilde{g_t}
\end{equation}

By the assumption $\mathcal{L}(\theta)$ is $L$-smooth, and take the expectation on randomness of step $t$ includes noise sampling and data sampling, we have:

\begin{align}\label{equasmooth}
&\mathbb{E}_t[\mathcal{L}(\theta_{t+1})]\\ \nonumber
&\leq \mathcal{L}(\theta_t)+\mathbb{E}_t[\langle g_t,(\theta_{t+1}-\theta_t)\rangle]+\mathbb{E}_t[\frac{L}{2}||\theta_{t+1}-\theta_t||^2]\\ \nonumber
&=\mathcal{L}(\theta_t)-\eta \langle g_t,\mathbb{E}_t[\tilde{g_t}]\rangle+\frac{L\eta^2}{2}\mathbb{E}_t[||\tilde{g_t}||^2]\\ \nonumber
&=\mathcal{L}(\theta_t)-\eta \langle g_t,\mathbb{E}_t[\bar{g_t}+\frac{1}{B}\mathcal{N}(0,\sigma^2C_t^2\mathbb{I})]\rangle+ \\ \nonumber
& \frac{L\eta^2}{2}\mathbb{E}_t[||\bar{g_t}+\frac{1}{B}\mathcal{N}(0,\sigma^2C_t^2\mathbb{I})||^2]\\ \nonumber
&\leq\mathcal{L}(\theta_t)-\eta \langle g_t,\mathbb{E}_t[\bar{g_t}]\rangle+L\eta^2\mathbb{E}_t[||\bar{g_t}||^2]+\frac{L\eta^2\sigma^2d}{B^2}C_t^2
\end{align}

In \Cref{equasmooth}, the last inequality follows from Cauchy Schwartz. And $d$ is the dimension of the gradient. Then sum up the above equation from $t=1$ to $t=T$ and take the expectation over the randomness across all the iterations, we get:
\begin{small}
\begin{equation}
\label{equasum}
\begin{split}
&\mathbb{E}[\mathcal{L}(\theta_{T+1})]-\mathbb{E}[\mathcal{L}(\theta_{1})]\\
&\leq-\eta\mathbb{E}[\sum_{t=1}^T\langle g_t,\mathbb{E}_t[\bar{g_t}]\rangle]+L\eta^2\sum_{t=1}^T\mathbb{E}[||\bar{g_t}||^2]+\frac{L\eta^2\sigma^2d}{B^2}\sum_{t=1}^T C_t^2
\end{split}
\end{equation}
\end{small}
Now we pay attention to $\mathbb{E}[\sum_{t=1}^T<g_t, \mathbb{E}_t[\bar{g_t}]>]$, rearranging \Cref{equasum} and divide it by $T\eta$. We get:

\begin{align}\label{equa14}
&\frac{1}{T}\mathbb{E}[\sum_{t=1}^T\langle g_t,\mathbb{E}_t[\bar{g_t}]\rangle]\\ \nonumber
&\leq\frac{1}{T\eta}(\mathbb{E}[\mathcal{L}(\theta_{1})]-\mathbb{E}[\mathcal{L}(\theta_{T+1})])+\frac{L\eta}{T}\sum_{t=1}^T\mathbb{E}[||\bar{g_t}||^2] \\ \nonumber
& +\frac{L\eta\sigma^2d}{B^2T}\sum_{i=1}^TC_t^2
\end{align}

We first handle $\mathbb{E}_t[\bar{g_t}]$:

\begin{equation}
\label{egt}
\begin{split}
&\mathbb{E}_t[\bar{g_t}]\\
&=\mathbb{E}_{\mathcal{B}_t}[\frac{1}{B}\sum_{i\in\mathcal{B}_t} {\rm Clip}(g_{t,i},C_t)]\\
&=\mathbb{E}_x[{\rm Clip}(g_{t,i},C_t)]\\
&=\frac{1}{N}\sum_{i=1}^N{\rm Clip}(g_{t,i},C_t)\\
&=\sum_{i=1}^N\frac{1}{N}{\rm Clip}(g_{t,i},C_t)\\
&=\sum_{i=1}^N\frac{1}{N}\mathbb{I}[||g_{t,i}||\leq C_t]g_{t,i}+\sum_{i=1}^N\frac{1}{N}\mathbb{I}[||g_{t,i}||>C_t]g_{t,i}\frac{C_t}{||g_{t,i}||}\\
&=\sum_{i=1}^N\frac{1}{N}\mathbb{I}[||g_{t,i}||\leq C_t]g_{t,i}+\sum_{i=1}^N\frac{1}{N}\mathbb{I}[||g_{t,i}||>C_t]g_{t,i}\\
&+\sum_{i=1}^N\frac{1}{N}\mathbb{I}[||g_{t,i}||>C_t]g_{t,i}(\frac{C_t}{||g_{t,i}||}-1)\\
&=\sum_{i=1}^N\frac{1}{N}g_{t,i}+\sum_{i=1}^N\frac{1}{N}\mathbb{I}[||g_{t,i}||>C_t]g_{t,i}(\frac{C_t}{||g_{t,i}||}-1)\\
&=g_t+\sum_{i=1}^N\frac{1}{N}\mathbb{I}[||g_{t,i}||>C_t]g_{t,i}(\frac{C_t}{||g_{t,i}||}-1)
\end{split}
\end{equation}

In \Cref{egt}, $\mathbb{I}()$ will output 1 when the input is true. By \assref{boundofg}, we have:

\begin{equation}
\label{eg16}
\begin{split}
&||\mathbb{E}[\bar{g_t}]-g_t||\\
&=||\sum_{i=1}^N\frac{1}{N}\mathbb{I}[||g_{t,i}||>C_t]g_{t,i}(\frac{C_t}{||g_{t,i}||}-1)||\\
&\leq \sum_{i=1}^N\frac{1}{N}\mathbb{I}[||g_{t,i}||>C_t](G-C_t)\\
&\leq \sum_{i=1}^N\frac{1}{N}\mathbb{I}[||g_{t,i}||>C_t]G\\
&=P(C_t)G
\end{split}
\end{equation}

The first inequality is due to Minkowski Inequality and \Cref{boundofg}. In \Cref{eg16}, we use $P(C_t)$ to represent the probability that a single gradient is being clipped at iteration $t$. In \dcsgdp, percentile $p$ reflects how much gradient will not be clipped, assuming our histogram publication is accurate enough, thereby $P(C_t)=1-p=p'$. Then we can get:

\begin{equation}
\label{finalequa}
\begin{split}
&\frac{1}{T}\mathbb{E}[\sum_{t=1}^T\langle g_t,\mathbb{E}_t[\bar{g_t}]\rangle]\\
&=\frac{1}{T}\mathbb{E}[\sum_{t=1}^T\langle g_t,(g_t-g_t+\mathbb{E}_t[\bar{g_t}])\rangle]\\
&\geq \frac{1}{T}\mathbb{E}[\sum_{t=1}^T||g_t||^2]-\frac{1}{T}\mathbb{E}[\sum_{t=1}^T||g_t||*||g_t-\mathbb{E}_t[\bar{g_t}]||]\\
&\geq \frac{1}{T}\mathbb{E}[\sum_{t=1}^T||g_t||^2]-\frac{1}{T}\mathbb{E}[\sum_{t=1}^T||g_t||p'G]
\end{split}
\end{equation}

The first inequality follows from Cauchy Schwartz. The second inequality is due to \Cref{eg16} Note that $g_t=\frac{1}{N}\sum_{i=0}^{N}g_{t,i}=\nabla\mathcal{L}(\theta_t)$, Combining \Cref{finalequa} and \Cref{equa14}, we get:

\begin{equation}
\label{final2}
\begin{split}
&\frac{1}{T}\mathbb{E}[\sum_{t=1}^T||\nabla\mathcal{L}(\theta_t)||^2] \\
&\leq\frac{1}{T\eta}(\mathbb{E}[\mathcal{L}(\theta_{1})]-\mathbb{E}[\mathcal{L}(\theta_{T+1})])+\frac{L\eta}{T}\sum_{t=1}^T\mathbb{E}[||\bar{g_t}||^2]\\
&+\frac{L\eta\sigma^2d}{B^2T}\sum_{i=1}^TC_t^2+\frac{1}{T}\mathbb{E}[\sum_{t=1}^T||g_t||p'G]
\end{split}
\end{equation}

\Cref{final2} is very similar to the result of \cite{du2021dynamic}. And we use the norm distribution and $p<1$ to decide $C$, we can get $C_t\leq G$ by \assref{boundofg}. Therefore, we get:

\begin{equation}
\label{final4}
\begin{split}
&\frac{1}{T}\mathbb{E}[\sum_{t=1}^T||\nabla\mathcal{L}(\theta_t)||^2] \\
&\leq\frac{1}{T\eta}(\mathbb{E}[\mathcal{L}(\theta_{1})]-\mathbb{E}[\mathcal{L}(\theta_{T+1})])+\frac{L\eta}{T}\sum_{t=1}^T\mathbb{E}[||\bar{g_t}||^2]\\
&+\frac{L\eta\sigma^2G^2d}{B^2}+\frac{p'}{T}\mathbb{E}[\sum_{t=1}^T||g_t||G]\\
\end{split}
\end{equation}

Note that the $\bar{g}$ is the gradient after clipped, we have $C_t\leq G$, and we can get $||g_t||\leq G$ with Assumption 2 and Cauchy Schwartz inequality. Therefore, we have $||\bar{g}_t||<G$. With \Cref{abadiMA}, let $\sigma= c_2\frac{q\sqrt{T\ln(1/\delta)}}{\epsilon}$, where $q=B/N$. With \assref{boundofl} (used to get the upper bound of $-\mathcal{L}(\theta_{t+1})$, remove the dependency of $\mathcal{L}(\theta_{T+1})$ on $T$). we have:

\begin{equation}
\label{final3}
\begin{split}
&\frac{1}{T}\mathbb{E}[\sum_{t=1}^T||\nabla\mathcal{L}(\theta_t)||^2] \\
&\leq \frac{1}{T\eta}(\mathbb{E}[\mathcal{L}(\theta_{1})]-\mathbb{E}[\mathcal{L}(\theta_{T+1})]) \\
&+L\eta G^2+\frac{LT\eta d\ln(1/\delta)v^2 G^2}{N^2\epsilon^2}+p'G^2 \\
&\leq \mathcal{O}(\frac{1}{T\eta})+\mathcal{O}(\eta)+\mathcal{O}(\frac{T\eta d \ln(1/\delta)}{N^2\epsilon^2})+\underbrace{p'G^2}_{bias} 
\end{split}
\end{equation}

The last non-vanishing term is due to bias. By setting $T\eta=\frac{N\epsilon}{\sqrt{d\ln(1/\delta)}}$ and $\eta=\frac{\sqrt{d\ln(1/\delta)}}{N\epsilon}$, we can further get:
\begin{equation}
\label{finalfinal}
\begin{split}
\frac{1}{T}\mathbb{E}[\sum_{t=1}^T||\nabla\mathcal{L}(\theta_t)||^2]\leq\mathcal{O}(\frac{\sqrt{d\ln(1/\delta)}}{N\epsilon})+{\rm bias}
\end{split}
\end{equation}

which finishes the proof.

\section{Convergence Analysis of \dcsgde}
\label{proofDCSE}
We follow the proof in \appref{proofP}, but we do not use the upper bound $G$ of gradient in \Cref{eg16}, and we use the actual value. Then modifying \Cref{final2}, and with \assref{boundofg}, we have:

\begin{equation}
\label{appendixDCE}
\begin{split}
&\frac{1}{T}\mathbb{E}[\sum_{t=1}^T||\nabla\mathcal{L}(\theta_t)||^2] \\
&\leq\frac{1}{T\eta}(\mathbb{E}[\mathcal{L}(\theta_{1})]-\mathbb{E}[\mathcal{L}(\theta_{T+1})])+\frac{L\eta}{T}\sum_{t=1}^T\mathbb{E}[||\bar{g_t}||^2]\\
&+\frac{L\eta\sigma^2d}{TB^2}\sum_{t=1}^TC_t^2+\frac{1}{T}\mathbb{E}[\sum_{t=1}^T||g_t||]*(\frac{1}{N}\sum_{i=1}^N \max(||g_{t,i}||-C_t,0))\\
&\leq \frac{1}{T\eta}(\mathbb{E}[\mathcal{L}(\theta_{1})]-\mathbb{E}[\mathcal{L}(\theta_{T+1})])+\frac{L\eta}{T}\sum_{t=1}^T\mathbb{E}[||\bar{g_t}||^2]\\
&+ \frac{1}{T}\sum_{t=1}^T(L\eta\underbrace{\frac{\sigma^2C_t^2d}{B^2}}_{Variance}+\underbrace{\frac{1}{N}\sum_{i=1}^N G\max(||g_{t,i}||-C_t,0)}_{Bias})
\end{split}
\end{equation}

The variance term and bias term are similar to our expected squared error in \dcsgde in the formulation. Therefore we can think \dcsgde balances the variance and bias in some sense, which may help the convergence.

Note that we choose $C_t$ as the threshold when $C_t$ is larger than all gradients, the bias term will be 0. Since the variance increases when $C_t$ increases and we will output $C_t$ with the smallest expected squared error, therefore $C_t$ is also bounded by the bound of the gradient. Therefore we can get the following:

\begin{equation}
\label{final1DCE}
\begin{split}
&\frac{1}{T}\mathbb{E}[\sum_{t=1}^T||\nabla\mathcal{L}(\theta_t)||^2] \\
&\leq \frac{1}{T\eta}(\mathbb{E}[\mathcal{L}(\theta_{1})]-\mathbb{E}[\mathcal{L}(\theta_{T+1})])+\frac{L\eta}{T}\sum_{t=1}^T\mathbb{E}[||\bar{g_t}||^2]\\
&+ \frac{1}{T}\sum_{t=1}^T(L\eta\frac{\sigma^2G^2d}{B^2}+G^2)
\end{split}
\end{equation}

Similar to \appref{proofP}, with the assumptions and set $\eta=\frac{1}{\sqrt{T}}$, we have:

\begin{equation}
\label{finalfinalDCSE}
\begin{split}
\frac{1}{T}\mathbb{E}[\sum_{t=1}^T||\nabla\mathcal{L}(\theta_t)||^2]\leq\mathcal{O}(\frac{\sqrt{d\ln(1/\delta)}}{N\epsilon})+{\rm bias}
\end{split}
\end{equation}

which finishes the proof.

\section{Datasets}
\noindent \textbf{Datasets.} We use four benchmark datasets to measure the privacy-utility trade-off: CIFAR10~\cite{krizhevsky2009learning}, MNIST~ \cite{lecun-mnisthandwrittendigit-2010}, SVHN~\cite{netzer2011reading}, and QNLI~\cite{wang2018glue}. The first three are for computer vision tasks, and the last one is for natural language inference tasks.

\begin{itemize}
\item \textbf{MNIST} \cite{lecun-mnisthandwrittendigit-2010} is a collection of handwritten digit images with 10 categories. It has a training set of 60,000 examples and a test set of 10,000 examples. Each example has a 28x28 grayscale image and a label from 0-9 digits.

\item \textbf{CIFAR10} \cite{krizhevsky2009learning} is acollection of colored objects images with 10 categories. It has a training set of 50,000 examples and a test set of 10,000 examples. Each example has a 32x32 colored image.

\item \textbf{SVHN} \cite{netzer2011reading} is a collection of printed digits cropped from pictures of house number plates with 10 categories. It has a training set of 73,257 examples and a test set of 26,032 examples. Each example has a 32x32 colored image.

\item \textbf{QNLI} \cite{wang2018glue} is a large collection of English ``problem"-``sentence" pairs manually labeled as entailment and not entailment for Natural Language Inference (NLI), which is to determine the relationship between two short texts. It has a training set of 104,743 examples and a test set of 5,463 examples.
\end{itemize}

\section{Hyperparameter Setting}

\begin{table}[htp]
	\centering
        \small
	\caption{Hyperparameter grids of different methods.}
        \label{parametergrid}
\begin{tabular}{c|c}
\toprule
Method        & Hyperparameter                       \\
\midrule
\dpsgd         & $C$, $\eta$ (SGD optimizer)           \\
\dcsgdp        & $p$, $\eta$ (SGD optimizer)           \\
\dcsgde        & $\eta$ (SGD optimizer)                \\
Andrew et al.\cite{andrew2021differentially} & $p$, $\eta_C$, $\eta$ (SGD optimizer) \\
Du et al. \cite{du2021dynamic}    & $C$, $\rho_C$,$\eta$ (SGD optimizer) \\
\bottomrule
\end{tabular}
\end{table}

\begin{table}[htp]
	\centering
        \small
	\caption{Candidates of hyperparameter to be tuned.}
        \label{parametercandidates}
\begin{tabular}{c|c}
\toprule
Hyperparameter & Candidates                          \\
\midrule
$C$            & \{0.1 0.2 0.5 0.8 1 2 4 6 8 10\}        \\
$\eta$         & \{$(2^i)\times10^{-5},i\in [0,19]$\}, totally 20         \\
$\eta_C$       & \{0.1 0.15 0.2 0.25 0.3\}              \\
$p$            & \{0.1 0.2 0.3 0.4 0.5 0.6 0.7 0.8 0.9\} \\
$\rho_C$        & \{0.1 0.2 0.3 0.4 0.5 0.6 0.7 0.8 0.9\} \\
\bottomrule
\end{tabular}
\end{table}

\section{More Results of Section VI-C}
\label{appRQ3}
This section presents more results on the influence of parameters.
Figure~\ref{fig:accp} shows the accuracy of different $p$ for \dcsgdp on CIFAR10 (ResNet18) and MNIST (CNN).
Figure~\ref{fig:accb} presents the accuracy of different $b$ for \dcsgdp (P) and \dcsgde (E) on CIFAR10 (ResNet18) and MNIST (CNN).
Figure~\ref{acce} shows the accuracy of different $\sigma_H$ for \dcsgdp (P) and \dcsgde (E) on CIFAR10 (ResNet18) and MNIST (CNN).

\begin{figure}[!h]
    \centering
    \subfloat{
        \includegraphics[width=1.0\linewidth]{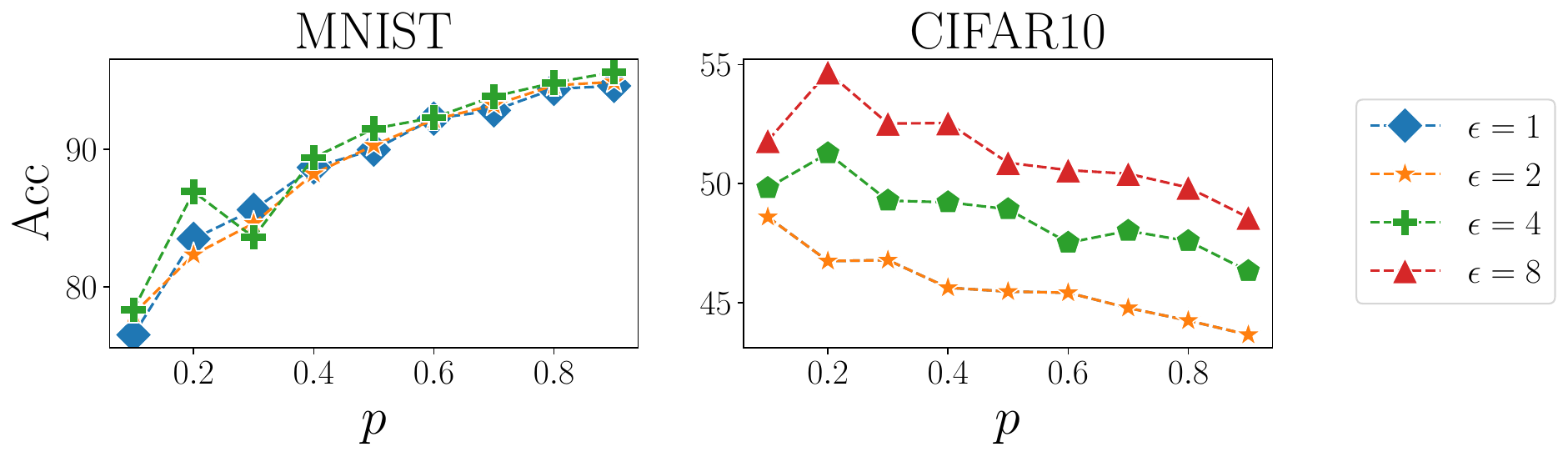}
    }
    \caption{The accuracy of different $p$ for \dcsgdp on CIFAR10 (ResNet18) and MNIST (CNN).}
    \label{fig:accp}
\end{figure}

\begin{figure}[!h]
    \centering
    \subfloat{
        \includegraphics[width=1.0\linewidth]{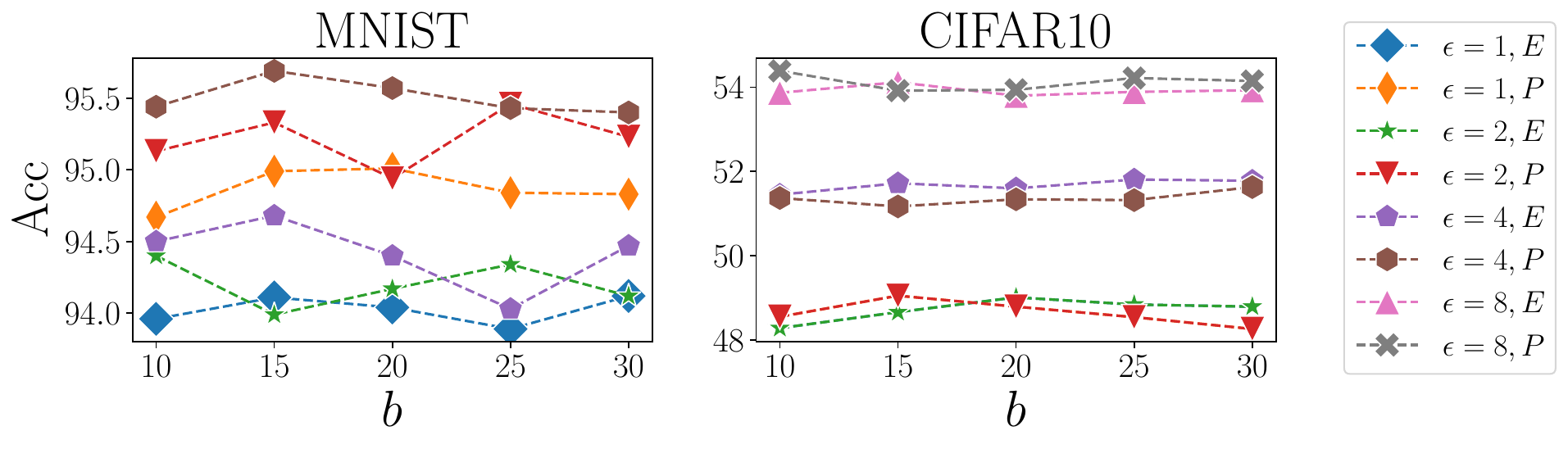}
    }
    \caption{The accuracy of different $b$ for \dcsgdp (P) and \dcsgde (E) on CIFAR10 (ResNet18) and MNIST (CNN).}
    \label{fig:accb}
\end{figure}

\begin{figure}[!h]
    \centering

    \subfloat{
        \includegraphics[width=1.0\linewidth]{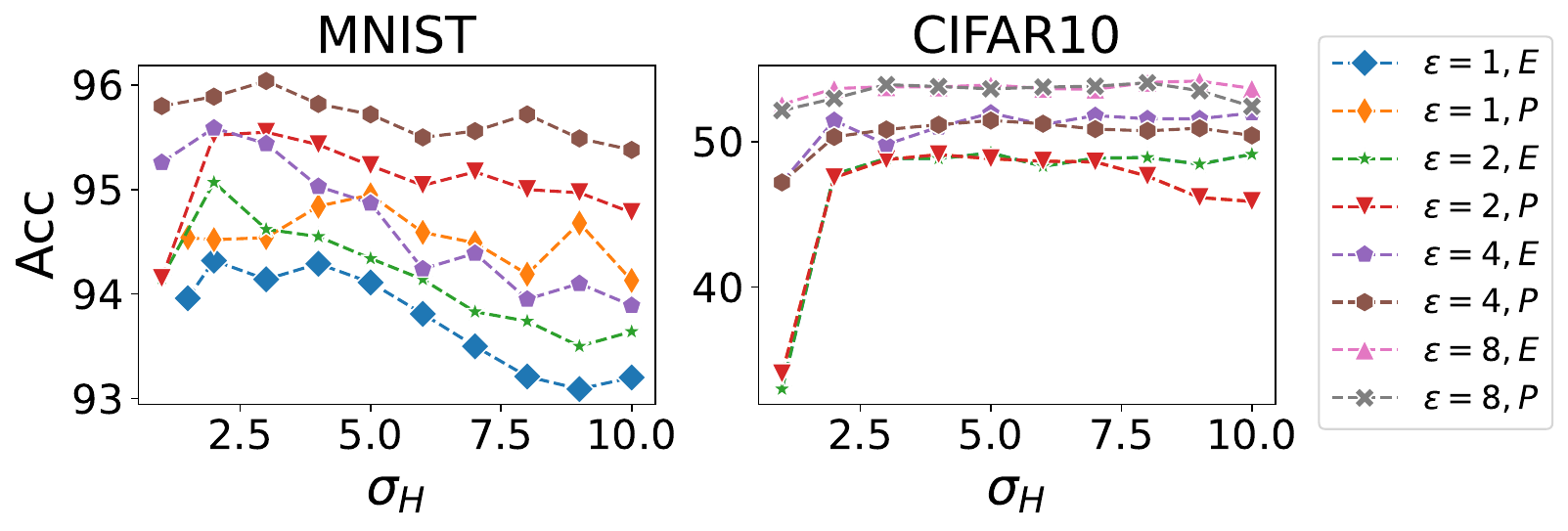}
    }
    \caption{The accuracy of different $\sigma_H$ for \dcsgdp (P) and \dcsgde (E) on CIFAR10 (ResNet18) and MNIST (CNN).}
    \label{acce}
\end{figure}


\section{Comparison with AutoClip}
In this section, we give the comparison result with AutoClip\cite{bu2022automatic}. We conduct the experiments with the same setting in main text on Adam. Because AutoClip provides a new clipping operation, we want to test its stability under different settings. Thus, we add a new model-dataset pair LSTM-NAMES (The model and dataset are like \url{https://pytorch.org/tutorials/intermediate/char_rnn_classification_tutorial.html}. NAMES is a dataset consisting of names from 18 countries; the task is identifying the country from which the name originates. And we use a model with two LSTM layers.). We also set epochs to 20 and use the default parameter setting of the Adam optimizer. Since AutoClip does not need to tune $C$ or other parameters, we only provide the result without the privacy cost of hyperparameter tuning.

The result is as Table \ref{AutoClipResult}, our methods significantly perform better on SVHN and NAMES and have a close performance under other settings. For autoclip, it performs excellent on QNLI and MNIST while imperfect on NAMES, and the difference is very significant; this may be due to its particular clipping operation. In summary, AutoClip gives a new clipping operation, which is not widely accepted. Though our work is orthogonal to AutoClip, our work can still achieve a better or close performance compared to AutoClip.

\begin{table*}[t]
	\centering
        \footnotesize
\setlength{\tabcolsep}{4mm}{
	\caption{The accuracy comparison of our methods and AutoClip for Adam optimizer without the privacy cost of hyperparameter tuning.}
        \label{AutoClipResult}
	\begin{tabular}{c|c|c|cccc}
        \toprule
		Models&Dataset&$\epsilon$&\dpsgd&\dcsgdp&\dcsgde&AutoClip \\
            \midrule
            \multirow{3}{*}{CNN} & \multirow{3}{*}{MNIST}&$1$&$94.78\pm0.21$&$94.72\pm0.19$&$94.04\pm0.30$&$\mathbf{96.54\pm0.10}$\\
            &&$2$&$95.62\pm0.15$&$95.08\pm0.27$&$94.19\pm0.87$&$\mathbf{96.83\pm0.11}$\\
            &&$4$&$95.95\pm0.11$&$95.60\pm0.20$&$94.65\pm0.16$&$\mathbf{97.21\pm0.16}$\\
            \midrule
            \multirow{3}{*}{BERT-base} & \multirow{3}{*}{QNLI}&$2$&$73.21\pm0.48$&$73.21\pm0.36$&$74.31\pm0.38$&$\mathbf{74.92\pm0.57}$\\
            &&$4$&$75.36\pm0.22$&$74.35\pm0.25$&$\mathbf{75.61\pm0.21}$&$75.52\pm0.33$\\
            &&$8$&$76.20\pm0.53$&$75.65\pm0.15$&$76.29\pm0.39$&$\mathbf{77.00\pm0.30}$\\
            
            \midrule
            \multirow{6}{*}{ResNet18} & \multirow{3}{*}{CIFAR10}&$2$&$48.34\pm0.45$&$48.63\pm0.42$&$\mathbf{48.89\pm0.48}$&$48.10\pm0.22$\\
            &&$4$&$51.28\pm0.16$&$\mathbf{51.46\pm0.15}$&$51.42\pm0.38$&$51.44\pm0.26$\\
            &&$8$&$53.81\pm0.46$&$\mathbf{54.18\pm0.32}$&$53.76\pm0.27$&$53.67\pm0.32$\\
            
            \cline{2-7}
            
            &\multirow{3}{*}{SVHN}&$2$&$80.05\pm0.21$&$80.42\pm0.28$&$\mathbf{81.39 \pm 0.46}$&$80.18\pm0.21$\\
            &&$4$&$81.03\pm0.46$&$81.43\pm0.22$&$\mathbf{82.45\pm0.36}$&$81.52\pm0.24$\\
            &&$8$&$81.92\pm0.46$&$83.06\pm0.13$&$\mathbf{83.44\pm0.40}$&$82.19\pm0.23$\\

            \midrule
            \multirow{6}{*}{ResNet34}
            &\multirow{3}{*}{CIFAR10}&$2$&$43.72\pm0.40$&$43.66\pm0.38$&$\mathbf{44.07\pm0.59}$&$43.29\pm0.28$\\
            &&$4$&$46.89 \pm 0.59$&$47.34\pm0.26$&$47.30\pm0.44$&$\mathbf{47.53\pm0.46}$\\
            &&$8$&$49.79\pm 0.36$&$49.28\pm0.44$&$50.10\pm0.28$&$\mathbf{50.26\pm0.29}$\\
            
            \cline{2-7}
            
            &\multirow{3}{*}{SVHN}&$2$&$78.47\pm0.33$&$78.90\pm0.49$&$\mathbf{79.26 \pm 0.28}$&$78.13\pm0.41$\\
            &&$4$&$79.70\pm0.47$&$80.08\pm0.18$&$\mathbf{81.67\pm0.34}$&$79.58\pm0.65$\\
            &&$8$&$80.66\pm0.30$&$81.81\pm0.39$&$\mathbf{82.47\pm0.46}$&$80.74\pm0.35$\\
            \midrule
            \multirow{3}{*}{LSTM} & \multirow{3}{*}{NAMES}&$2$&$55.14\pm0.25$&$\mathbf{55.55\pm0.38}$&$55.26\pm0.30$&$54.48\pm0.23$\\
            &&$4$&$57.24\pm1.32$&$\mathbf{59.56\pm0.73}$&$58.56\pm0.88$&$55.22\pm0.18$\\
            &&$8$&$59.91\pm0.22$&$\mathbf{60.96\pm0.41}$&$60.71\pm0.53$&$56.34\pm0.94$\\
        
        \bottomrule
	\end{tabular}
	}
 \vspace{-0.3cm}
\end{table*}

\end{document}